\documentclass[sigconf]{acmart}
\usepackage{bbm}
\usepackage{amsmath}
\usepackage{amsthm}
\usepackage{amsfonts}
\usepackage{multirow}
\usepackage{algorithm}
\usepackage{algorithmic}
\usepackage{bm} 
\usepackage{enumitem}
\usepackage{graphicx}
\usepackage{balance}
\usepackage{subfigure} 
\newcommand{\myPara}[1]{\vspace{.05in}\noindent\textbf{#1}}

\newtheorem{theorem}{Theorem}
\newtheorem{lemma}{Lemma}
\newtheorem{definition}{Definition}
\newtheorem{assumption}{Assumption}
\def \xx {\bm{x}}
\def \zz {\bm{z}}

\AtBeginDocument{%
  \providecommand\BibTeX{{%
    \normalfont B\kern-0.5em{\scshape i\kern-0.25em b}\kern-0.8em\TeX}}}

\copyrightyear{2022} 
\acmYear{2022} 
\setcopyright{rightsretained} 

\acmConference[MM '22]{Proceedings of the 30th ACM International Conference on Multimedia}{October 10--14, 2022}{Lisboa, Portugal}
\acmBooktitle{Proceedings of the 30th ACM International Conference on Multimedia (MM '22), Oct. 10--14, 2022, Lisboa, Portugal}
\acmDOI{10.1145/3503161.3547984}
\acmISBN{978-1-4503-9203-7/22/10}

\usepackage{etoolbox}
\makeatletter
\patchcmd{\maketitle}{\@copyrightpermission}{
   \begin{minipage}{0.3\columnwidth}
     \href{https://creativecommons.org/licenses/by/4.0/}{\includegraphics[width=0.90\textwidth]{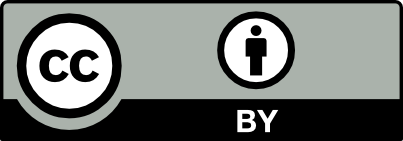}}
   \end{minipage}\hfill
   \begin{minipage}{0.7\columnwidth}
     \href{https://creativecommons.org/licenses/by/4.0/}{This work is licensed under a Creative Commons Attribution International 4.0 License.}
   \end{minipage}
  
   \vspace{5pt}
}{}{}

\makeatother

\begin{document}

\title{Tackling Instance-Dependent Label Noise with Dynamic Distribution Calibration}

\author{Manyi Zhang}
\email{zhang-my21@mails.tsinghua.edu.cn}
\affiliation{
\institution{SIGS, Tsinghua University}
\city{Shenzhen}
\country{China}
}

\author{Yuxin Ren}
\email{ryx20@mails.tsinghua.edu.cn}
\affiliation{
\institution{SIGS, Tsinghua University}
\city{Shenzhen}
\country{China}
}

\author{Zihao Wang}
\email{wangziha21@mails.tsinghua.edu.cn}
\affiliation{
\institution{SIGS, Tsinghua University}
\city{Shenzhen}
\country{China}
}

\author{Chun Yuan}
\authornote{Corresponding author.}
\email{yuanc@sz.tsinghua.edu.cn}
\affiliation{
\institution{SIGS, Tsinghua University}
\institution{Peng Cheng National Laboratory}
\city{Shenzhen}
\country{China}
}


\begin{abstract}
 Instance-dependent label noise is realistic but rather challenging, where the label-corruption process depends on instances directly. It causes a severe \textit{distribution shift} between the distributions of training and test data, which impairs the generalization of trained models. Prior works put great effort into tackling the issue. Unfortunately, these works always highly rely on strong assumptions or remain heuristic without theoretical guarantees. In this paper, to address the distribution shift in learning with instance-dependent label noise, a \textit{dynamic distribution-calibration} strategy is adopted. Specifically, we hypothesize that, before training data are corrupted by label noise, each class conforms to a \textit{multivariate Gaussian distribution} at the feature level. Label noise produces outliers to shift the Gaussian distribution. During training, to calibrate the shifted distribution, we propose two methods based on the \textit{mean} and \textit{covariance} of multivariate Gaussian distribution respectively. The mean-based method works in a recursive dimension-reduction manner for robust mean estimation, which is theoretically guaranteed to train a high-quality model against label noise. The covariance-based method works in a distribution disturbance manner, which is experimentally verified to improve the model robustness. We demonstrate the utility and effectiveness of our methods on datasets with synthetic label noise and real-world unknown noise.
 
\end{abstract}

\ccsdesc[500]{Computing methodologies~Computer vision; Machine learning}

\keywords{Instance-dependent label noise, distribution shift, distribution calibration, robustness}


\maketitle

\section{Introduction}\label{sec:1}
Learning with label noise is one of the hottest topics in weakly-supervised learning \cite{han2020sigua,lukasik2020does,collier2021correlated}. In real life, large-scale datasets are likely to contain label noise. The main reason is that manual high-quality labeling is expensive \cite{northcutt2021confident,ortego2021multi,wu2021ngc}. Large-scale datasets are always collected from crowdsourcing platforms \cite{li2017learning} or crawled from the internet \cite{yao2020searching}, which inevitably introduces label noise.
\begin{figure}[!t]
\centering
\includegraphics[width=0.44\textwidth]{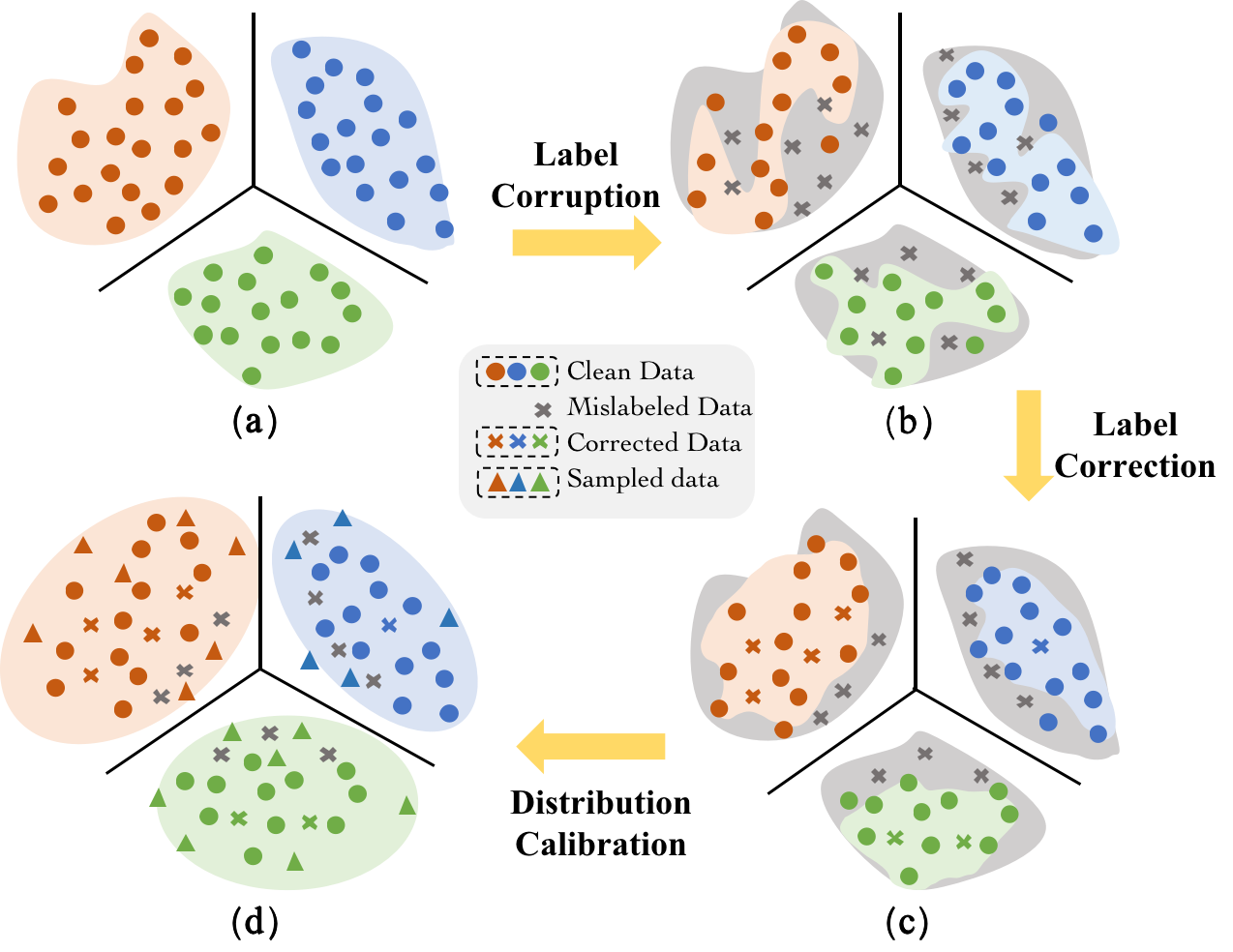} 
\caption{The illustrations of the distribution-shift problem and our solution. The background color areas represent the data distributions of different classes. (a) This represents clean data and ground-truth data distributions. (b) Due to label corruption, partial data are mislabeled, resulting in noisy labels. The distributions trained on data with noisy labels mismatch to ground-truth distributions. (c) Label correction mainly corrects data far from decision boundaries. However, the mislabeled data near decision boundaries are still not corrected, making learned distributions biased. (d) By using our distribution calibration, the learned distributions are closer to ground-truth distributions, naturally following better generalization abilities.}
\vspace{-10pt}
\label{fig:comparison} 
\end{figure}

Instance-dependent label noise \cite{cheng2017learning,xia2020part,cheng2020learning} is \textit{more realistic} and \textit{applicable} than instance-independent label noise, where the label-flipping process depends on instances/features directly. It is because, in real-world scenes, an instance whose features contain less discriminative information or are of poorer quality may be more likely to be mislabeled. Instance-dependent label noise is more challenging owing to its inherent complexity~\cite{zhu2021second}. Compared with instance-independent label noise, it leads to a more severe \textit{distribution shift} problem for trained models \cite{berthon2021idn}. That is to say, this kind of noise makes the distributions of training and test data \textit{significantly different}. If the models are trained with instance-dependent label noise, it is pessimistic that they would generalize poorly on test data~\cite{xia2021sample}. 

The recent methods on handling instance-dependent label noise are generally divided into two main
categories. The first one is to estimate the instance-dependent transition matrix \cite{xia2020part,jiang2022an,xia2022extended}, which tends to characterize the behaviors of clean labels flipping into noisy labels. However, these methods are limited to the case with a small number of classes \cite{zhu2021second,xia2019anchor}. Besides, they highly rely on strong assumptions to achieve an accurate estimation, \textit{e.g.}, the assumptions on anchor points, bounded noise rates, and extra trusted data. It is hard or even infeasible to check these assumptions, which hinders the validity of these methods~\cite{jiang2022an}. The second one tends to heuristically identify clean data based on the memorization effect of deep models \cite{arpit2017closer} for subsequent operations, \textit{e.g.}, label correction~\cite{tanaka2018joint,zheng2020error,tanaka2018joint}. Unfortunately, due to the complexity of instance-dependent label noise, label correction will be much weaker in the noisy region near the decision boundary. Therefore, the labels corrected by the current predictions would likely be erroneous~\cite{berthon2021idn}. In addition, the training data in identified clean regions in this way is relatively monotonous~\cite{xia2021sample,pleiss2020identifying}. The corresponding distribution will be restricted to a small region of the global distribution, which introduces covariate shift\footnote{Covariate shift is a subclass of distribution shift. It is when the distribution of instances shifts between training and test environments. Although the instance distribution may change, the labels remain the same. }~\cite{sugiyama2012machine}. We detail the above issues in Figure~\ref{fig:comparison}. Therefore, the methods belonging to both two categories cannot well handle the distribution shift brought by instance-dependent label noise. 

In this paper, to address the above issues, we propose a dynamic distribution-calibration strategy. We first assume that, before training data are corrupted by label noise, each class conforms to a \textit{multivariate Gaussian distribution} at the feature level. Such an assumption is more reasonable and has been verified in lots of works, such as~\cite{hernandez2014mind,kendall2017uncertainties}. Then, two methods based on the mean and covariance of multivariate Gaussian distributions are proposed. Specifically, the mean-based method works in a \textit{recursive dimension-reduction} manner, which is theoretically guaranteed to 
train a high-quality model against instance-dependent label noise.  
It first assigns smaller weights to outliers and then divides the whole feature space into \textit{clean space} and \textit{corrupted space} where the contamination has larger effects on corrupted space. We recurse the computation on corrupted space for robust mean estimation.
The covariance-based method works in a \textit{distribution disturbance} manner, which is experimentally verified. It introduces an interference in the empirical covariance of given data. In this way, we can increase the diversity of training data to mitigate the distribution shift and improve generalization. After achieving multivariate Gaussian distributions for all classes, we sample examples from them for training, which calibrates the shifted distributions. 

We conduct extensive experiments across various settings on \textit{CIFAR-10}, \textit{CIFAR-100}, \textit{WebVision}, and \textit{Clothing1M}. The results consistently exhibit substantial performance improvements compared to state-of-the-art methods, which support our claims well.

\myPara{Organization.} The rest of this paper is organized as follows. In Section~\ref{sec:2}, we introduce some background knowledge. In Section~\ref{sec:3}, we present our methods step by step, with theoretical justifications. In Section~\ref{sec:4}, empirical evaluations are provided. In Section~\ref{sec:5}, we summarize this paper. 

\begin{figure}[!tp]
\centering
\includegraphics[width=0.4\textwidth]{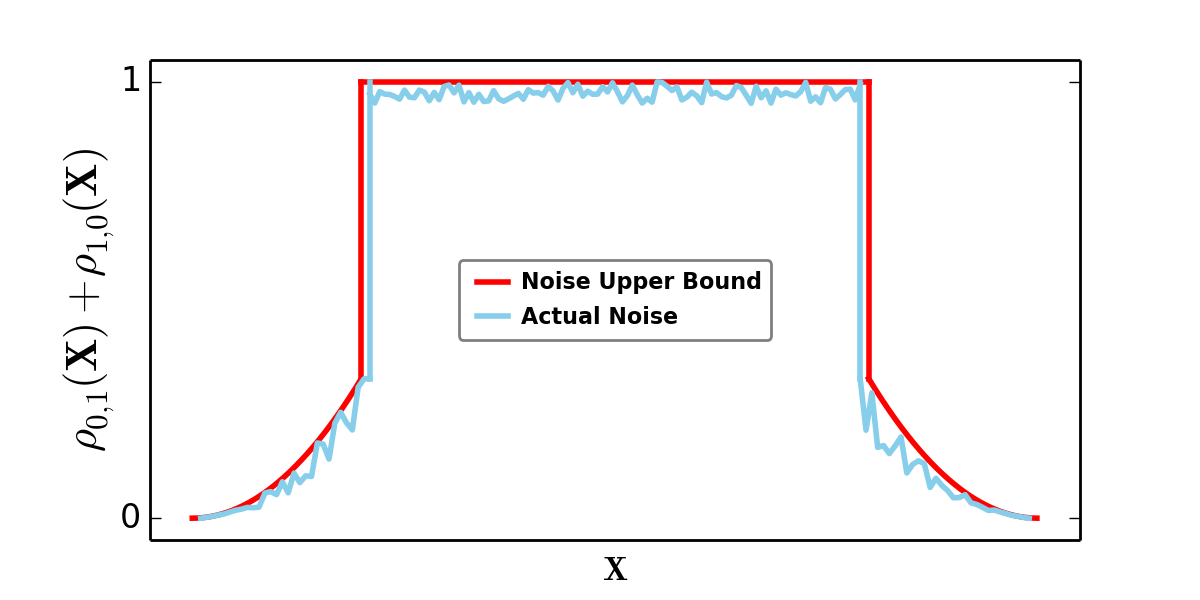} 
\caption{The noise level curve of PMD noise. The red curve represents the upper bound. The noise level is bounded by $\rho$ in the restricted region while unbounded in the unrestricted region. The blue curve shows the actual noise level.}
\label{fig:PMD_curve}
\vspace{-5pt}
\end{figure}
\section{Preliminaries}\label{sec:2}
\myPara{Notations.} Let $\mathbbm{I}[\mathcal{A}]$ be the indicator of the event $\mathcal{A}$. Let $[z]=\{0,\ldots,z-1\}$. Besides, $|\mathcal{B}|$ denotes the total number of elements in the set $\mathcal{B}$. 

\myPara{Problem setup.} We consider a $k$-classification problem ($k\geq 2$). Let $\mathcal{X}$ and $\mathcal{Y}=[k]$ be the instance and class label spaces respectively. We assume the dataset $\{(\xx_i,y_i)\}_{i=1}^n$ is sampled from the underlying joint distribution $\mathcal{D}$ over $\mathcal{X}\times\mathcal{Y}$, where $n$ is the sample size. Before observation, partial clean labels are flipped due to instance-dependent label noise. As a result, we are provided with a noisy training dataset $\{(\xx_i,\tilde{y}_i)\}_{i=1}^n$ obtained from a noisy joint distribution $\tilde{\mathcal{D}}$ over $\mathcal{X}\times\mathcal{Y}$. For each instance $\xx_i$, its label $\tilde{y}_i$ may be incorrect. Our goal is to learn a robust classifier by \textit{only} exploiting the noisy dataset, which can assign clean labels to test data precisely. 

\myPara{Label noise model.} In this paper, we consider \textit{polynomial-margin diminishing noise} (\textbf{PMD} noise)~\cite{zhang2021learningwith}. The PMD noise is one of the instance-dependent label noise, which is realistic but rather challenging. 
Although the noise setting and analyses naturally generalize to the multi-class case ($k>2$), for simplicity, we first focus on the binary case ($k=2$) to improve legibility. Specifically, let 
$\eta(\xx)=\mathbbm{P}(y=1|\xx)$ be the clean class posterior. Let $\rho_{0,1}(\xx)=\mathbbm{P}[\tilde{y}=1|y=0,\xx]$ and $\rho_{1,0}(\xx)=\mathbbm{P}[\tilde{y}=0|y=1,\xx]$ be the noise functions where $\tilde{y}$ denotes the noisy label. For example, an instance $\xx$ has the true label $y=0$, the noisy label is flipped to $1$ with the  probability of $\rho_{0,1}(\xx)$. We give the formal definition of PMD noise as follows. 
\begin{definition}[PMD noise]
The noise functions $\rho_{0,1}(\xx)$ and $\rho_{1,0}(\xx)$ are PMD, if there are constants $0<t_0<\frac{1}{2}$ and $c_1,c_2>0$ such as 
\begin{align}
   &\rho_{0,1}(\xx)\leq c_1\eta(\xx)^{1+c_2};\forall\eta(\xx)\leq\frac{1}{2}-t_0; \\
    &\rho_{1,0}(\xx)\leq c_1[1-\eta(\xx)]^{1+c_2};\forall\eta(\xx)\geq\frac{1}{2}+t_0. 
\end{align}

\end{definition}
We denote $t_0$ as the margin of $\rho$. In the restricted region $\{\xx:|\eta(\xx)-\frac{1}{2}|>t_0\}$, the PMD noise requires the \textit{upper bound} on $\rho$ to be polynomial and monotonically decreasing~\cite{zhang2021learningwith}. Meanwhile, in the unrestricted region $\{\xx:|\eta(\xx)-\frac{1}{2}|<t_0\}$, both $\rho_{0,1}$ and $\rho_{1,0}$ can be arbitrary. For a clearer understanding, Figure~\ref{fig:PMD_curve} provides the noise level curve of the PMD noise. The PMD noise is consistent with real-world noisy scenarios: data near the decision boundary is hard to distinguish and likely to be mislabeled so the noise level can be arbitrary. Meanwhile, data far away from the decision boundary owns typical features of its class so the noise level should be bounded. 

\section{Methodology}\label{sec:3}
In this section, we present our methods step by step. We first show how to extract the deep features of instances based on label correction and how to model deep features (Section \ref{sec:3.1}). Then, we detail the mean-based method (Section \ref{sec:3.2}) and covariance-based method (Section \ref{sec:3.3}). Lastly, the theoretical analysis is provided to justify our claims (Section \ref{sec:3.4}).  

\subsection{Feature Extraction and Label Correction}\label{sec:3.1}
Given a training example $(\xx,\tilde{y})$, we employ a deep network to finish the encoding and classification for it. In more detail, first, the instance $\xx$ is fed into the deep network to obtain its deep features, which are denoted by $h(\xx)\in\mathbbm{R}^d$. Then, we map the deep features to the class label space with the classifier $f(h(\xx))$. 
For simplicity, we use $f(\xx)$ to substitute $f(h(\xx))$. In binary classification, $f(x):\mathcal{X}\rightarrow[0,1]$. Let $y_{f(\xx)}:= \mathbbm{I}(f(\xx)\geq \frac{1}{2}) $ be the label predicted by $f$. 

In this paper, we build our methods based on label correction~\cite{zhang2021learningwith}, mainly considering its sample-efficiency~\cite{tanaka2018joint}. Specifically, we first warm up the deep network within a few epochs. The warm-up phase is used to initialize the model for the following operations. Then, we correct the label, where $f$ has high confidence in it. With a threshold $\tau$, if $|f(\xx)-\frac{1}{2}|>\tau$, we flip $\tilde{y}$ to the prediction $y_{f(\xx)}$ of $f$. We repeatedly correct labels and improve the network within the epoch, until no label can be corrected. In the next epoch, we sightly relax the threshold $\tau$ for label correction. The above procedure can be easily extended to the \textit{multi-class} scenario (see Appendix B for more details).  

Recall that we assume that before training data
are corrupted by label noise, the features of each class conform to a multivariate Gaussian distribution. That is, for $k$ classes, we define $k$ multivariate Gaussian distributions in total. The distribution for the $c$-th class is denoted by $\mathcal{N}_c(h(\xx)|\bm{\mu}^c,\bm{\Sigma}^c)$, where $\bm{\mu}^c \in \mathbbm{R}^d$ and $\bm{\Sigma}^c \in \mathbbm{R}^{d\times d}$.

Note that when all deep feature distributions are ideal, \textit{i.e.,} they are not corrupted by label noise,  the \textit{empirical mean} $\bar{\bm{\mu}}^c$ is well-known to be an optimal estimator of the true mean $\bm{\mu}^c$ at most $\mathcal{O}(\sqrt{d/n_c)}$, where $n_c$ is the number of examples belonging to the $c$-th class. The existence of label noise makes the empirical mean fail: even a single corrupted data can arbitrarily mislead the mean estimation. However, as discussed, the label correction approach cannot handle distribution shift caused by instance-dependent label noise in two aspects: (1) The labels corrected by the current predictions would likely be erroneous; (2) The training data in identified clean regions in this way is relatively monotonous, leading to covariate shift. The two issues motivate us to explore more robust solutions.

\subsection{Mean-Based Method}\label{sec:3.2}
Based on the Huber's contamination model~\cite{huber1992robust} and our assumption, we introduce the algorithm \textit{AgnosticMean}~\cite{lai2016agnostic}. In summary, the algorithm consists of two steps: (1) the outlier damping step; (2) the projection step. In the outlier damping step,  we assign different weights to each data point where outliers will get smaller weights. In the projection step, we project data points onto the span (denoted as $\mathcal{V}$) of the top $\frac{d}{2}$ principal components. Two steps are \textit{alternately implemented} to achieve robust mean estimation, which mitigates the side effect of instance-dependent label noise. In the following, we give mathematical descriptions of our mean-based method. 

Given the noisy dataset $\{(\xx_i,\tilde{y}_i)\}_{i=1}^n$, after label correction at the $t$-th epoch, we obtain the dataset $\{(\xx_i,\tilde{y}_i^t)\}_{i=1}^n$. For building the multivariate Gaussian distribution on deep features $h(\xx)$, we take instances whose labels are $c$ as examples, the rest classes are the same procedure. 
We group these instances into a set $\mathcal{S}$.
Due to the corruption caused by label noise, we have $\mathcal{S} = \mathcal{S}_\mathcal{N} \cup \mathcal{S}_\mathcal{Q}$, where $\mathcal{S}_\mathcal{N}$ and $\mathcal{S}_\mathcal{Q}$ include the instances sampled from the underlying multivariate Gaussian distribution and arbitrary noise distribution respectively. If the noise rate is $\epsilon$, we have $|\mathcal{S}_\mathcal{Q}| = \epsilon|\mathcal{S}|$. We assume that the underlying distribution is $\mathcal{N}(h(\xx)|\bm{\mu},\sigma^2\bm{I})$. Note that the assumption on the covariance is practical in statistical/machine learning and related tasks (c.f.~ \cite{kingma2013auto,kingma2016improved,diakonikolas2019recent}). The notations $\bm{\mu}_\mathcal{S}$ (\textit{resp}. $\bm{\mu}_\mathcal{Q}$) and $\bm{\Sigma}_{\mathcal{S}}$ (\textit{resp}. $\bm{\Sigma}_{\mathcal{Q}}$) denote the mean and covariance of the set $\mathcal{S}$ (\textit{resp}. $\mathcal{S}_{\mathcal{Q}}$). We then have 
\begin{equation}
    \bm{\Sigma}_{\mathcal{S}}=(1-\epsilon)\sigma^2\bm{I}+\epsilon\bm{\Sigma}_{\mathcal{Q}}+\epsilon(1-\epsilon)(\bm{\mu}_\mathcal{\mathcal{S}}-\bm{\mu}_\mathcal{Q})(\bm{\mu}_\mathcal{S}-\bm{\mu}_\mathcal{Q})^\top.
\end{equation}
Our main goal is to find the mean shift, \textit{i.e.,} $\bm{\mu}_\mathcal{S}-\bm{\mu}$. The rationality behind this is that if outliers substantially move $\bm{\mu}_\mathcal{S}$ far away from $\bm{\mu}$, it must move $\bm{\mu}_\mathcal{S}$ in the direction of $\bm{\mu}_\mathcal{S}-\bm{\mu}$. The outlier damping step can effectively limit the effect of outliers. The projection step contributes to find the direction of the mean shift.

\myPara{Outlier damping step.} In this step, we impose different weights on different instances. Specifically, the instances that are far away from the \textit{coordinate-wise median} are endowed with smaller weights. Here, the coordinate-wise median simply takes the median in each dimension. Subsequently, we use $\bm{m}$ to denote the coordinate-wise median of $\mathcal{S}$. Let $s^2=C\textsf{Tr}(\bm{\Sigma}_{\mathcal{S}})$, where $C$ is a constant, $\textsf{Tr}(\bm{\Sigma}_{\mathcal{S}})$ is the trace of $\bm{\Sigma}_{\mathcal{S}}$. The weight of each instance can be obtained by the following equation:
\begin{equation}\label{eq:weights}
    \omega_i = \exp(-\frac{||h(\xx_i)-\bm{m}||_2^2}{s^2}),  \xx_i \in \mathcal{S}.
\end{equation}
After obtaining different weights of all instances, we reset the covariance of the set $\mathcal{S}$ with the weights. That is, 
\begin{equation}\label{eq:weighted_covariance}
    \bm{\Sigma}_{\mathcal{S},\bm{\omega}}=\sum\omega_i\frac{(h(\xx_i)-\bm{\mu}_{\mathcal{S},\bm{\omega}})(h(\xx_i)-\bm{\mu}_{\mathcal{S},\bm{\omega}})^\top}{n_c},
\end{equation}
where $\bm{\mu}_{\mathcal{S},\bm{\omega}}=\sum\omega_i\frac{h(\xx_i)}{n_c}$ and $n_c$ denotes the number of examples belonging to the $c$-th class. 

\myPara{Projection step.}  
With the aim of containing a large projection of $\bm{\mu}_\mathcal{S}-\bm{\mu}$, we project $\mathcal{S}$ onto the span $\mathcal{V}$ of the top $\frac{d}{2}$ principal components of $\bm{\Sigma}_\mathcal{S, \bm{\omega}}$. For the span of the bottom $\frac{d}{2}$ principal components (denoted as $\mathcal{W}$).        
 $\mathcal{P}_\mathcal{W}\bm{\mu}$ is just approximated by $\mathcal{P}_\mathcal{W}\bm{\mu}_\mathcal{S}$, where $\mathcal{P}_\mathcal{W}$ represents the operation projecting to the space $\mathcal{W}$. After determining the projection of $\bm{\mu}$ in space $\mathcal{W}$, we still need to ensure the mean in space $\mathcal{V}$. We recursively perform the outlier damping step and projection step in $\mathcal{V}$. To the end, there will be only a one-dimensional space which is treated as the direction of $\bm{\mu}_\mathcal{S}-\bm{\mu}$. Note that when the dimension of deep features is 1, the median of $\mathcal{S}$ is known to be a robust estimator of the mean with the error of $\mathcal{O}(\epsilon\sigma)$~\cite{huber2004robust,hampel2011robust}.  So we just compute the median as the robust mean estimator.

Figure~\ref{fig:projection} illutrates the effectiveness of the mean-based method. After the outlier damping step, outliers are assigned with smaller weights. The mean shift $\bm{\mu}_\mathcal{S}-\bm{\mu}$ onto the space $\mathcal{W}$ has very small projection, which implies that the contribution of outliers in space $\mathcal{W}$ on average cancels out and the sample mean leads to a small error. Therefore, we just need to perform the outlier damping step and projection step in space $\mathcal{V}$ at the next procedure until there is only one dimension.    
The procedure of the algorithm \textit{AgnosticMean} is presented in Algorithm~\ref{alg:robust_mean}. Emphatically, Steps~4-6 belong to outlier damping. Step~7 belongs to projection. In Step~8, the outlier damping and projection steps are performed recursively, until there will be only a one-dimensional space. The procedure of \textit{AgnosticMean} is executed for any $c-$th class to obtain its robust mean $\hat{\bm{\mu}}^c$. We will provide theoretical insights for the mean-based method in Section~\ref{sec:3.4}.

\begin{algorithm}[!t]
    \caption{\textit{AgnosticMean}}
    \label{alg:robust_mean}
    \begin{algorithmic}[1]
        \REQUIRE The set $\mathcal{S}=\{h(\xx_1),\ldots,h(\xx_{n_c})\}$ with $h(\xx_{i})\in\mathbbm{R}^d$.
        \IF{$d = 1$ }
            \RETURN \textsf{median}($\mathcal{S}$).
        \ENDIF
        \STATE \textbf{Let} $s^2 = C \textsf{Tr}(\bm{\Sigma}_\mathcal{S})$;
        \STATE \textbf{Compute} the weight $\omega_i$ for every $h(\bm{x_i})\in\mathcal{S}$ with Eq.~(\ref{eq:weights}), $\bm{\omega}=(\omega_1, ..., \omega_{n_c})$;
        \STATE \textbf{Reset} the covariance of the set $\mathcal{S}$ as  $\bm{\Sigma}_{\mathcal{S},\bm{\omega}}$ with Eq.~(\ref{eq:weighted_covariance});
        \STATE \textbf{Project} $\mathcal{S}$ onto the span $\mathcal{V}$ and $\mathcal{W}$ to obtain $\mathcal{P}_\mathcal{V}\mathcal{S}$ and $\mathcal{P}_\mathcal{W}\mathcal{S}$, $\mathcal{P}_\mathcal{V}\mathcal{S}\in  \mathbbm{R}^\frac{d}{2}$, $\mathcal{P}_\mathcal{W}\mathcal{S}\in  \mathbbm{R}^\frac{d}{2}$;
        \STATE \textbf{Set} $\bm{\hat{\mu}}_\mathcal{V}$ = \textit{AgnosticMean}($\mathcal{P}_\mathcal{V}\mathcal{S}$) recursively and $\bm{\hat{\mu}}_\mathcal{W}$ = \textsf{mean}($\mathcal{P}_\mathcal{W}\mathcal{S}$);
        \STATE \textbf{Obtain} $\bm{\hat{\mu}} \in \mathbbm{R}^d$, where $\mathcal{P}_\mathcal{V} \bm{\hat{\mu}} = \bm{\hat{\mu}}_\mathcal{V}$ and $\mathcal{P}_\mathcal{W} \bm{\hat{\mu}} = \bm{\hat{\mu}}_\mathcal{W}$.
    \end{algorithmic}
\end{algorithm}
\begin{figure}[!t]
\centering
\includegraphics[width=0.21\textwidth]{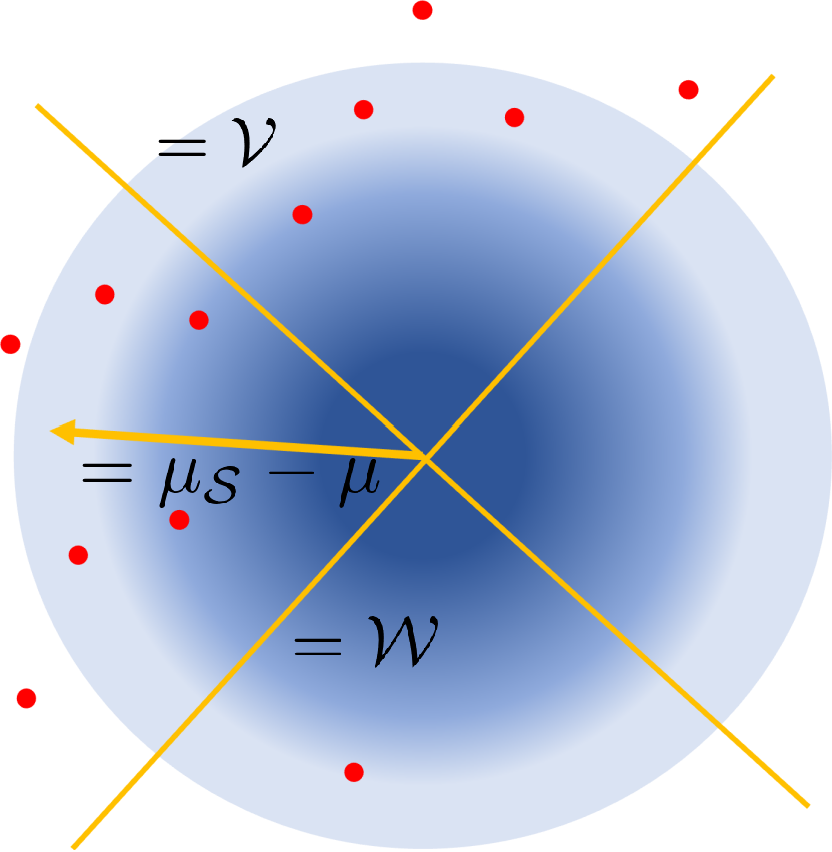} 

\caption{The illustration of the mean-based method. The circular area is the data space where the color depth degree indicates the weights of data. The points in the darker region will be assigned larger weights. The red points are outliers and the orange arrow represents the direction of the mean shift.}
\vspace{-10pt}
\label{fig:projection}
\end{figure}

\subsection{Covariance-Based Method}\label{sec:3.3}
As discussed, if we exploit the memorization effect of deep networks~\cite{arpit2017closer}, the training data in identified clean regions in this way is relatively monotonous, leading to covariate shift. In other words, the training data in identified clean regions are sampled from a biased distribution and cannot cover the underlying distribution well~\cite{wang2019implicit,li2022uncertainty,yang2021free}. 

To relieve the above issue, we add some disturbance to the empirical covariance of the biased distribution. Intuitively, in this way, we can increase the diversity of training data to mitigate the overfitting to the biased distribution. Specifically, for any $c$-th class, we first compute the mean and covariance of the corresponding distribution of the $c$-th class, \textit{i.e.}, 
\begin{align} \label{eq:sample_estimator}
&\hat {\bm{\mu}}  =  \sum \frac{h(\bm{x}_i)}{n_c},\notag\\
&\hat{\bm{\Sigma}} = \sum\frac{\left(h(\bm{x}_i) - \hat {\bm{\mu}}\right) \left(h(\bm{x}_i) - \hat {\bm{\mu}}\right)^\top}{n_c}.
\end{align}
Then, we add the disturbance to the covariance $\hat{\bm{\Sigma}}$, which can expand the features' regions, rendering the model gains robustness against covariate shift. That is, 

\begin{align} \label{eq:sample_estimator_dis}
\hat{\bm{\Sigma}}^* = \sum\frac{\left(h(\bm{x}_i) - \hat {\bm{\mu}}\right) \left(h(\bm{x}_i) - \hat {\bm{\mu}}\right)^\top}{n_c}+\alpha\mathbf{1},
\end{align}
where $\alpha\in\mathbbm{R}^+$ is a hyperparameter that controls the degree of the disturbance, and $\mathbf{1}\in\mathbbm{R}^{d\times d}$ is a matrix of ones. Note that it is rather hard to set an optimal $\alpha$ from a theoretical view. It is because, without strong assumptions, the determination of the optimal covariance is an open problem in robust statistics. Therefore, in this paper, we determine the value of $\alpha$ by simply employing a (noisy) validation set. Experimental results prove the feasibility of this way. 

\myPara{Algorithm flow.} We present the overall paradigm of our methods in Algorithm \ref{alg:main}. For the setting of $\tau(t)$, we follow the setting in \cite{zhang2021learningwith}. In summary, we first use warm-up training to make the model predictions more reliable (Step~1). Then, in each epoch, we exploit label correction to obtain a corrected dataset (Steps~3-7) for the following extraction of deep features (Step~8). Next, we perform the proposed two methods (described in Section \ref{sec:3.2} and \ref{sec:3.3}) to build multivariate Gaussian distributions for all classes (Steps~9-11). At last, extra data points are sampled from the built multivariate Gaussian distributions for training, which finishes the distribution calibration to improve generalization (Steps~12-15). In this work, we name our methods with label correction as \textbf{m}ean-based \textbf{d}ynamic \textbf{d}istribution \textbf{c}alibration (MDDC) and \textbf{c}ovariance-based \textbf{d}ynamic \textbf{d}istribution \textbf{c}alibration (CDDC) respectively. 
\begin{algorithm}[!t]
    \caption{The algorithm flow of our methods}
    \label{alg:main}
    \begin{algorithmic}[1]
        \REQUIRE The noisy dataset $\{(\xx_i,\tilde{y}_i)\}_{i=1}^n$, the maximum of the training epoch $t_{\max}$, the number of warm up $t_{w}$, the threshold $\tau(t)$, the initial classifier $f$, and the sampled fraction $\lambda$. 
        \STATE \textbf{Warm} up the classifier $f$ on $\{(\xx_i,\tilde{y}_i)\}_{i=1}^n$ within $t_w$ epochs;
        \FOR{ $t = t_w+1, ..., t_{max}$}
                \FORALL{$(\xx_i,\tilde{y}_i^{t-1})\in \{(\xx_i,\tilde{y}_i^{t-1})\}_{i=1}^n$ and $|f(\xx_i)-\frac{1}{2}|\geq\tau(t)$}
                    \STATE $\tilde{y}_i^t \leftarrow \mathbbm{1}_{\{f(\xx_i)\geq \frac{1}{2}\}}$\hfill //label correction;
                \ENDFOR
                \STATE \textbf{Obtain} the corrected dataset $\{(\xx_i,\tilde{y}_i^{t})\}_{i=1}^n$;
                \STATE \textbf{Train} $f$ on the corrected dataset;
                \STATE \textbf{Extract} the deep features of all instances, \textit{i.e.}, $h(\xx_i)$;
                \FOR{$c$ = 1, ..., $k$}
                    \STATE \textbf{Build} multivariate Gaussian distributions for the $c$-th class using the mean-based or covariance-based method; 
                \ENDFOR
                \STATE \textbf{Sample} extra examples from all $k$ multivariate Gaussian distributions with the number of $\lambda n$\hfill //data sampling;
                \STATE \textbf{Train} $f$ on the sampled examples;
        \ENDFOR
        \RETURN Trained classifier $f$.
    \end{algorithmic}
\end{algorithm}

\subsection{Theoretical Analysis}\label{sec:3.4}
We provide theoretical analyses of MDDC based on PLC~\cite{zhang2021learningwith}, which shows that a classifier trained with our strategy converges to be consistent with the Bayes optimal classifier with high probability. In this work, we focus on the asymptotic case~\cite{mohri2018foundations}.

\myPara{Pre-knowledge.} For completeness, we first reproduce the relevant background knowledge. 
\begin{definition}[$(\beta, v)$-consistency~\cite{zhang2021learningwith}]\label{def:consistency_beta_v}
Suppose that a set of data $(\bm{x},\tilde{y})$ are sampled from $\tilde{\mathcal{D}}(\bm{x},\tilde{\eta}(\bm{x}))$, where $\tilde{\eta}(\bm{x})$ outputs the noisy posterior probability for $\bm{x}$. Given a hypothesis class $ \mathcal{M}$ with sufficient examples, and $f(\xx)= \mathop{\arg\min}\limits_{m\in \mathcal{M}}  \mathbb{E}_{(\xx,\tilde{y})\sim\tilde{\mathcal{D}}(\bm{x},\tilde{\eta}(\bm{x}))}Loss(m(x),\tilde{y})$, we define $\mathcal{M}$ is $(\beta, v)$-consistency if
\begin{align}\label{eq:consistency_beta_v}
    &|f(\xx)-\eta(\xx)|\leq\\
    &\beta\mathbbm{E}_{(\zz, \tilde{y})\sim\tilde{\mathcal{D}}(\zz,\tilde{\eta}(\zz))}\big[\mathbbm{1}_{\{ \tilde{y}_{\zz}\neq\eta^*(\zz)\}}(\zz)\big||\eta(\zz)-\frac{1}{2}|\geq|\eta(\xx)-\frac{1}{2}|\big]+v \notag,
\end{align}
where $\eta(\zz)$ outputs the clean posterior probability for $\zz$ and $\eta^*(\zz)$ is the Bayes optimal classifier for $\zz$. 
\end{definition}
Definition~\ref{def:consistency_beta_v} measures the inconsistency between $f(\cdot)$ and $\eta(\cdot)$. We can approximate the error of the classifier $f$ at $\xx$ by computing the risk of $\eta^*(\cdot)$ at $\bm{z}$ where $\eta^*(\zz)$ are more confident than $\eta^*(\xx)$.

\begin{definition}[$c_*,c^*)$-bounded distribution~\cite{zhang2021learningwith}]
Denote the cumulative distribution function of $|\eta(\xx)-\frac{1}{2}|$ as $R(\cdot)$. For a random variable $o$, $R(o)= \mathbbm{P}(|\eta(\xx)-\frac{1}{2}|\leq o)$ and the corresponding probability density function is $r(o)$. We define the distribution $\mathcal{D}$ is ($c_*,c^*)$-bounded if $0< c_*\leq r(o) \leq c^*$ for all $0\leq o \leq \frac{1}{2}$. The worst-case density-imbalance ratio of $\mathcal{D}$ is denoted by $\ell:=\frac{c^*}{c_*}$.
\end{definition}

The bounded condition enforces the continuity of the density function. The continuity allows one example to borrow useful information from its (clean) neighborhood region, which can help handle mislabeled data. 

\begin{definition}[Pure ($\tau, f, \eta$)-level set]
We say a set $S(\tau, \eta) :=\big \{\xx \big||\eta(\xx)-\frac{1}{2}|\geq\tau\big \}$ is pure for the classifier $f$, if $y_{f(\xx)}=\eta^*(\xx)$ for all $\xx\in S(\tau, \eta)$, where $y_{f(\xx)}$ denotes the label
predicted by $f$, \textit{i.e.}, $y_{f(\xx)}=\mathbbm{1}_{\{f(\xx)\geq\frac{1}{2}\}}$.
\end{definition}
The above definition indicates that the classifier $f$ has sufficient discriminating ability in a specific region.  

\myPara{Assumption.} We present the following assumption for theoretical results.

\begin{assumption}
    We assume that the given hypothesis class $\mathcal{M}$ is $(\beta, v)$-consistency, and the underlying distribution $\mathcal{D}$ satisfied the $(c^*, c_*)$-bounded condition. 
\end{assumption}

\myPara{Theoretical results.} Based on the above pre-knowledge and assumption, we formally present our theoretical results. 

\begin{lemma}[One round purity improvement]\label{lem:one_round}
Suppose that the classifier $f$ owns a pure ($\tau, f, \eta$)-level set where $3\xi v \leq \tau<\frac{1}{2}$, where $\xi<1$ is a constant that depends on the PMD noise. After one round label correction (described in Section~\ref{sec:3.1}) and one round sampling from estimated Gaussian distributions (described in Sections~\ref{sec:3.2} and~\ref{sec:3.3}), the training set consists of two parts: the corrected set ($\xx$, $\tilde{y}_{new})$ and the sampling set ($\xx$, $\tilde{y}_{\mathcal{G}}$). Training on these data, $f$ owns a improved pure ($\tau_{new}, f, \eta$)-level set where $\frac{1}{2}-\tau_{new}\geq (1+\frac{\xi v}{\beta \ell})(\frac{1}{2}-\tau) $. 
\end{lemma}
\begin{theorem}\label{thm:main}
    After training on the noisy dataset, the proposed method will return the trained classifier $f$ such that
\begin{equation}
    \mathbb{P}_{\xx\sim\tilde{\mathcal{D}}}[y_{f(\xx)}=\eta^*(\xx)]\geq 1-3\xi c^*v.
\end{equation}
\end{theorem} 
The proofs are provided in Appendix A. Lemma~\ref{lem:one_round} tells us that the cleansed region will be enlarged by at least a constant factor after one training round. Theorem~\ref{thm:main} states that the classifier trained with our method converges to be
consistent with the Bayes optimal classifier with high probability. Note that since $\xi<1$, $1-3\xi c^*v>1-3c^*v$ holds, which indicates our method can help the classifier converge to be
consistent with the Bayes optimal classifier with a higher probability than PLC~\cite{zhang2021learningwith}. The theoretical results verify the effectiveness of our method well. In the next section, we use comprehensive experiments to support our claims.

\section{Experiments}\label{sec:4}
In this section, we first introduce the used baselines in this paper (Section~\ref{sec:4.1}). The evaluations on synthetic and real-world label-noise datasets are then provided (Section~\ref{sec:4.2} and \ref{sec:4.3}). Finally, a comprehensive ablation study is provided (Section~\ref{sec:4.4}).

\subsection{Baselines}\label{sec:4.1}
In this paper, we compare our methods with five representative baselines and implement all methods with default parameters by PyTorch. The baselines include: (i). Standard; (ii). Co-teaching+~\cite{yu2019does}; (iii). GCE~\cite{zhang2018generalized}; (iv). SL~\cite{wang2019symmetric}; (v). LRT~\cite{zheng2020error}. The details of the above baselines are provided in Appendix C. 

As our methods finish distribution calibration after label correction using PLC~\cite{zhang2021learningwith}, we compare our methods with PLC later (Section~\ref{sec:4.3} and Section~\ref{sec:4.4}). In this way, we can demonstrate the improvement brought by our methods more systematically and clearly. Note that we do not directly compare the proposed methods with some state-of-the-art methods, \textit{e.g.}, DivideMix~\cite{li2020dividemix}. It is because, DivideMix is an aggregation of multiple advanced techniques, \textit{e.g.}, Mixup, soft labels, and semi-supervised learning. Therefore, the comparison is not fair. In this paper, to make a fair comparison, as did in \cite{xia2021sample}, we will boost DivideMix with our methods. More details will be provided shortly.

\begin{table*}[!ht]
    \centering
    \caption{Mean and standard deviations of test accuracy (\%) on two synthetic noisy datasets. The best two mean results are in \textbf{bold}.}
    \begin{tabular}{c|l|ccccc|cc}
    \hline
         Dataset & Noise&Standard&Co-teaching+&GCE&SL&LRT&MDDC&CDDC \\\hline
         
         \parbox[t]{2mm}{\multirow{6}{*}{\rotatebox{90}{\textit{CIFAR-10}}}}&{$\text{Type-\uppercase\expandafter{\romannumeral1}}$ (35\%)} & 78.11$\pm$0.74 & 79.97$\pm$0.15 & 80.65$\pm$0.39 & 79.76$\pm$0.72 & 80.98$\pm$0.80 & \textbf{83.12$\pm$0.44} & \textbf{83.17$\pm$0.43}\\
         &{$\text{Type-\uppercase\expandafter{\romannumeral1}}$ (70\%)}&41.98$\pm$1.96&40.69$\pm$1.99&36.52$\pm$1.62&36.29$\pm$0.66&41.52$\pm$4.53 & \textbf{43.70$\pm$0.91} & \textbf{45.14$\pm$1.11}\\\cline{2-9}
         &{$\text{Type-\uppercase\expandafter{\romannumeral2}}$ (35\%)} & 76.65$\pm$0.57&77.35$\pm$0.44&77.60$\pm$0.88&77.92$\pm$0.89&\textbf{80.74$\pm$0.25}&80.59$\pm$0.19&\textbf{80.76$\pm$0.25}\\
         &{$\text{Type-\uppercase\expandafter{\romannumeral2}}$ (70\%)}&45.57$\pm$1.12&45.44$\pm$0.64&40.30$\pm$1.46&41.11$\pm$1.92&44.67$\pm$3.89 & \textbf{47.06$\pm$0.41} & \textbf{48.10$\pm$1.74}\\\cline{2-9}
         &{$\text{Type-\uppercase\expandafter{\romannumeral3}}$ (35\%)}&76.89$\pm$0.79 & 78.38$\pm$0.67 & 79.18$\pm$0.61 & 78.81$\pm$0.29 & 81.08$\pm$0.35 & \textbf{82.02$\pm$0.23} & \textbf{81.58$\pm$0.07}\\
         &{$\text{Type-\uppercase\expandafter{\romannumeral3}}$ (70\%)} & 43.32$\pm$1.00 & 41.90$\pm$0.86 & 37.10$\pm$0.59 & 38.49$\pm$1.46 & 44.47$\pm$1.23 & \textbf{46.31$\pm$0.92} & \textbf{47.92$\pm$1.28}\\\hline\hline
         \parbox[t]{2mm}{\multirow{6}{*}{\rotatebox{90}{\textit{CIFAR-100}}}}&{$\text{Type-\uppercase\expandafter{\romannumeral1}}$ (35\%)}&57.68$\pm$0.29&56.70$\pm$0.71&58.37$\pm$0.18&55.20$\pm$0.33&56.74$\pm$0.34 & \textbf{63.12$\pm$0.40} & \textbf{62.89$\pm$0.20}\\
         &{$\text{Type-\uppercase\expandafter{\romannumeral1}}$ (70\%)}&39.32$\pm$0.43&39.53$\pm$0.28&40.01$\pm$0.71&40.02$\pm$0.85&45.29$\pm$0.43 & \textbf{48.32$\pm$0.85} & \textbf{47.78$\pm$0.53}\\\cline{2-9}
         &{$\text{Type-\uppercase\expandafter{\romannumeral2}}$ (35\%)}&57.83$\pm$0.25&56.57$\pm$0.52&58.11$\pm$1.05&56.10$\pm$0.73&57.25$\pm$0.68 & \textbf{63.84$\pm$0.23} & \textbf{64.57$\pm$0.41}\\
         &{$\text{Type-\uppercase\expandafter{\romannumeral2}}$ (70\%)}&39.30$\pm$0.32&36.84$\pm$0.39&37.75$\pm$0.46&38.45$\pm$0.45 & 43.71$\pm$0.51 & \textbf{48.41$\pm$0.96} & \textbf{47.74$\pm$0.62}\\\cline{2-9}
         &{$\text{Type-\uppercase\expandafter{\romannumeral3}}$ (35\%)}&56.07$\pm$0.79&55.77$\pm$0.98&57.51$\pm$1.16&56.04$\pm$0.74&56.57$\pm$0.30& \textbf{63.60$\pm$0.33} & \textbf{63.46$\pm$0.23}\\
         &{$\text{Type-\uppercase\expandafter{\romannumeral3}}$ (70\%)}&40.01$\pm$0.38&35.37$\pm$2.65&40.53$\pm$0.60&39.94$\pm$0.84&44.41$\pm$0.19 & \textbf{47.42$\pm$0.17} & \textbf{47.04$\pm$0.34}\\\hline
    \end{tabular}
    \label{tab:syn_balanced_dataset}
\end{table*}

\begin{table*}[!ht]
    \centering
    \caption{Mean and standard deviations of test accuracy (\%) on two balanced noisy datasets. Here, the label noise is mixed by instance-dependent and class-dependent noise. The noise levels of three kinds of instance-dependent label noise are 35\%. The real noise levels range from 50\% to 70\% lastly, since class-dependent noise is \textit{overlayed} on the instance-dependent one. The best two mean results are in \textbf{bold}.}
    \begin{tabular}{c|l|ccccc|cc}
    \hline
         Dataset & Noise&Standard&Co-teaching+&GCE&SL&LRT&MDDC&CDDC \\\hline
         \parbox[c]{2mm}{\multirow{9}{*}{\rotatebox{90}{\textit{CIFAR-10}}}}&
         {$\text{Type-\uppercase\expandafter{\romannumeral1}}$ + 30\% Sym.} & 75.26$\pm$0.32&78.72$\pm$0.53&78.08$\pm$0.66&77.79$\pm$0.46&75.97$\pm$0.27& \textbf{81.35$\pm$0.13} & \textbf{80.66$\pm$0.09}\\
         &{$\text{Type-\uppercase\expandafter{\romannumeral1}}$ + 60\% Sym.} & 64.25$\pm$0.78 & 55.49$\pm$2.11 & 67.43$\pm$1.43 & 67.63$\pm$1.36 & 59.22$\pm$0.74 & \textbf{74.13$\pm$0.26} & \textbf{74.26$\pm$0.58} \\
         &{$\text{Type-\uppercase\expandafter{\romannumeral1}}$ + 30\% Asym.} & 75.21$\pm$0.64 & 75.43$\pm$2.96 & 76.91$\pm$0.56 & 77.14$\pm$0.70 & 76.96$\pm$0.45 & \textbf{78.85$\pm$0.35} & \textbf{79.12$\pm$0.21}\\
         \cline{2-9} &{$\text{Type-\uppercase\expandafter{\romannumeral2}}$ + 30\% Sym.} & 74.92$\pm$0.63 & 75.19$\pm$0.54 & 75.69$\pm$0.21 & 75.08$\pm$0.47 & 75.94$\pm$0.58 & \textbf{79.01$\pm$0.32} & \textbf{79.31$\pm$0.57} \\
         &{$\text{Type-\uppercase\expandafter{\romannumeral2}}$ + 60\% Sym.} & 64.02$\pm$0.66 & 59.89$\pm$0.63 & 66.39$\pm$0.29 & 66.76$\pm$1.60 & 58.99$\pm$1.43 & \textbf{73.21$\pm$0.85} & \textbf{72.37$\pm$0.22}\\
         &{$\text{Type-\uppercase\expandafter{\romannumeral2}}$ + 30\% Asym.} & 74.28$\pm$0.39 & 73.37$\pm$0.83 & 75.30$\pm$0.81 & 75.43$\pm$0.42 & 77.03$\pm$0.62 & \textbf{77.46$\pm$0.24} & \textbf{78.04$\pm$0.25} \\\cline{2-9}
         &{$\text{Type-\uppercase\expandafter{\romannumeral3}}$ + 30\% Sym.} & 74.00$\pm$0.38 & 77.31$\pm$0.11 & 77.00$\pm$0.12 & 76.22$\pm$0.12 & 75.66$\pm$0.57 & \textbf{80.27$\pm$0.17} & \textbf{80.26$\pm$0.33}\\
         &{$\text{Type-\uppercase\expandafter{\romannumeral3}}$ + 60\% Sym.} & 63.96$\pm$0.69 & 56.78$\pm$1.56 & 67.53$\pm$0.51 & 67.79$\pm$0.54 & 59.36$\pm$0.93 & \textbf{73.96$\pm$0.42} &\textbf{73.30$\pm$0.50}\\ 
         &{$\text{Type-\uppercase\expandafter{\romannumeral3}}$ + 30\% Asym.} & 75.31$\pm$0.34 & 74.62$\pm$1.71 & 75.70$\pm$0.91 & 76.09$\pm$0.10 & 77.79$\pm$0.14 & \textbf{78.66$\pm$0.29} & \textbf{78.60$\pm$0.51}\\\hline\hline
        \parbox[c]{2mm}{\multirow{9}{*}{\rotatebox{90}{\textit{CIFAR-100}}}} & {$\text{Type-\uppercase\expandafter{\romannumeral1}}$ + 30\% Sym.} & 48.86$\pm$0.56 & 52.33$\pm$0.64 & 52.90$\pm$0.53 & 51.34$\pm$0.64 & 45.66$\pm$1.60 &\textbf{60.01$\pm$0.69} & \textbf{59.41$\pm$0.78}\\
         &{$\text{Type-\uppercase\expandafter{\romannumeral1}}$ + 60\% Sym.}&35.97$\pm$1.12&27.17$\pm$1.66&38.62$\pm$1.65&37.57$\pm$0.43&23.37$\pm$0.72 &\textbf{45.45$\pm$0.24} & \textbf{47.98$\pm$0.56} \\
         &{$\text{Type-\uppercase\expandafter{\romannumeral1}}$ + 30\% Asym.} & 45.85$\pm$0.93 & 51.21$\pm$0.31 & 52.69$\pm$1.14 & 50.18$\pm$0.97 & 52.04$\pm$0.15 & \textbf{61.70$\pm$0.06} & \textbf{60.78$\pm$0.61}\\
         \cline{2-9}
         &{$\text{Type-\uppercase\expandafter{\romannumeral2}}$ + 30\% Sym.}&49.32$\pm$0.36&51.99$\pm$0.75&53.61$\pm$0.46&50.58$\pm$0.25&43.86$\pm$1.31 &\textbf{60.50$\pm$0.18} & \textbf{59.03$\pm$0.28}\\
         &{$\text{Type-\uppercase\expandafter{\romannumeral2}}$ + 60\% Sym.} & 35.16$\pm$0.05&25.91$\pm$0.64&39.58$\pm$3.13&37.93$\pm$0.22&23.05$\pm$0.99& \textbf{44.67$\pm$0.78} &\textbf{47.41$\pm$0.48}\\
        &{$\text{Type-\uppercase\expandafter{\romannumeral2}}$ + 30\% Asym.} & 46.50$\pm$0.95 & 51.07$\pm$1.44 & 51.98$\pm$0.37 & 49.46$\pm$0.23 & 52.11$\pm$0.46 & \textbf{62.21$\pm$0.06} & \textbf{61.56$\pm$0.39}\\\cline{2-9}
         &{$\text{Type-\uppercase\expandafter{\romannumeral3}}$ + 30\% Sym.} & 48.94$\pm$0.61&49.94$\pm$0.44&52.07$\pm$0.35&50.18$\pm$0.54&42.79$\pm$1.78 & \textbf{60.03$\pm$0.41} & \textbf{58.61$\pm$0.24}\\
         &{$\text{Type-\uppercase\expandafter{\romannumeral3}}$ + 60\% Sym.}&34.67$\pm$0.16&22.89$\pm$0.75&36.82$\pm$0.49&37.65$\pm$1.42&22.81$\pm$0.72 & \textbf{46.44$\pm$0.86} & \textbf{47.70$\pm$0.10}\\ &{$\text{Type-\uppercase\expandafter{\romannumeral3}}$ + 30\% Asym.} & 45.70$\pm$0.12 & 49.38$\pm$0.86 & 50.87$\pm$1.12 & 48.15$\pm$0.90 & 50.31$\pm$0.39 & \textbf{62.23$\pm$0.24} & \textbf{60.92$\pm$0.25} \\\hline
    \end{tabular}
    \label{tab:syn_balanced_dataset_mixed}
\end{table*}

\subsection{Evaluations on Synthetic Datasets}\label{sec:4.2}
\myPara{Datasets.} We validate our methods on the manually corrupted version of \textit{CIFAR-10} \cite{krizhevsky2009learning} and \textit{CIFAR-100} \cite{krizhevsky2009learning}, which are popularly used in learning with label noise~\cite{wei2021optimizing,huang2019o2u,lee2018cleannet,kim2019nlnl,lee2019robust,nishi2021augmentation,kim2021fine,sachdeva2021evidentialmix,chen2020beyond,feng2021can}. \textit{CIFAR-10} has 10 classes of images including 50,000 training images and 10,000 test images. \textit{CIFAR-100} also has 50,000 training images and 10,000 test images, but 100 classes.

\myPara{Noise generation.} We consider a generic family of noise. Following \cite{zhang2021learningwith}, we employ not only instance-dependent noise, but also hybrid noise that includes both instance-dependent noise and class-dependent noise. For instance-dependent noise, we consider three kinds of noise functions within the PMD noise family. To make the task challenging enough, for input $\xx$, we flip the label from the most confident class $a_{\xx}$ to the second confident class $b_{\xx}$. The three noise functions are in the following:
\begin{equation}
\begin{aligned}
    &\text{Type-\uppercase\expandafter{\romannumeral1}}:\rho_{a_{\xx},b_{\xx}}=-\frac{1}{2}[\eta_{a_{\xx}}(\xx)-\eta_{b_{\xx}}(\xx)]^2+\frac{1}{2},\\
    &\text{Type-\uppercase\expandafter{\romannumeral2}}:\rho_{a_{\xx},b_{\xx}}=1-[\eta_{a_{\xx}}(\xx)-\eta_{b_{\xx}}(\xx)]^3,\\
    &\text{Type-\uppercase\expandafter{\romannumeral3}}:\rho_{a_{\xx},b_{\xx}}=1-\frac{1}{3}\big[[\eta_{a_{\xx}}(\xx)-\eta_{b_{\xx}}(\xx)]^3\\
    &+[\eta_{a_{\xx}}(\xx)-\eta_{b_{\xx}}(\xx)]^2+[\eta_{a_{\xx}}(\xx)-\eta_{b_{\xx}}(\xx)]\big].\\
\end{aligned}
\end{equation}
For PMD noise only, the noise levels are set to 35\% and 70\% respectively. For class-dependent noise, we exploit symmetric (abbreviated as sym.)~\cite{liu2019peer,liu2020early,thekumparampil2018robustness,lyu2020Curriculum,li2021provably} and asymmetric (abbreviated as asym.) noise~\cite{patrini2017making,yi2019probabilistic,wei2020combating,tan2021co,han2018masking,yao2021jo}. The noise is generated by the class-dependent transition matrix $T_{ij}=\mathbbm{P}(\tilde{y}=j|y=i)$ \cite{hendrycks2018using,wu2020class2simi,zhang2021learning}. If the noise level is $\epsilon$, for the symmetric one, the clean label is uniformly flipped to other classes, \textit{i.e.}, $T_{ij}=\epsilon/(k-1)$ for $i\neq j$, and $T_{ii}=1-\epsilon$. For the asymmetric one, the clean label $i$ is flipped to $j$ or remains unchanged with $T_{ij}=\epsilon$ or $T_{ii}=1-\epsilon$. In this paper, for \textit{CIFAR-10}, asymmetric noise is generated by mapping \texttt{TRUCK} $\rightarrow$ \texttt{AUTOMOBINE}, \texttt{BIRD} $\rightarrow$ \texttt{AIRPLANE}, \texttt{DEER} $\rightarrow$ \texttt{HORSE}, and \texttt{CAT}$\leftrightarrow$ \texttt{DOG}. For \textit{CIFAR-100}, the 100 classes are divided into 20 super-classes with each has 5 sub-classes, and we flip between two randomly selected sub-classes within each super-class. We leave out 10\% of noisy examples for validation. Note that the clean labels are dominating in noisy classes and that noisy labels are random, the accuracy on the noisy validation set and the accuracy on the clean test set are \textit{positively correlated}. The noisy validation set can thus be used~\cite{seo2019combinatorial,xia2021robust,nguyen2020self}.

\myPara{Model and configurations.} For \textit{CIFAR-10} and \textit{CIFAR-100}, we use a preact ResNet-34 network. We perform data
augmentation by horizontal random flips and 32×32 random crops after padding 4 pixels on each side. The batch size is set to 128 and we run 100 epochs \textit{CIFAR-10}, and 180 epochs for \textit{CIFAR-100}. We adopt SGD optimizer (momentum=0.9). It should be noted that, since the data sampled from Gaussian distribution are penultimate layer's features, we adopt two SGD optimizers for the corrected dataset and the new sampled dataset. The initialized learning rates are set to 0.01 and 0.0001 for two SGD optimizers. We divide learning rates by 0.5 at $40-$th and $80-$th epochs. All experiments are repeated three times with different random seeds. We report the mean and standard deviation of experimental results.
 
\myPara{Measurement.} To measure the performance, we use the test
accuracy, i.e. \textit{test accuracy = (\# of correct predictions) / (\# of testing)}. Intuitively, the higher test accuracy means that a method is more robust to label noise.

\myPara{Experimental results.}
Table~\ref{tab:syn_balanced_dataset} shows the results of our methods and other baselines under three types of PMD noise with noise levels 35\% and 70\%. For \textit{CIFAR-10}, as can be seen, our methods, \textit{i.e.}, MDDC and CDDC, outperform baselines in almost all cases. Although the baseline LRT can achieve comparable performance in the case $\text{Type-\uppercase\expandafter{\romannumeral2}}~(35\%)$, its performance decreases drastically as the noise rate increases. For \textit{CIFAR-100}, we observe that our methods exhibit a more distinct improvement. In addition, in partial noise settings, our methods can achieve over 7\% higher test accuracy than all baselines. Table~\ref{tab:syn_balanced_dataset_mixed}  summarizes the test accuracy achieved on the datasets corrupted with a combination of PMD noise and class-dependent noise. As can be seen, our methods work significantly better than baselines in all cases. In more detail, for \textit{CIFAR-100}, our methods can achieve more than 10\% lead over baselines in the case of combining PMD noise and additional 60\% symmetric noise.   

\subsection{Experiments on Real-World Datasets}\label{sec:4.3}
\myPara{Datasets.} In this paper, we exploit two large-scale real-world label-noise datasets, \textit{i.e.}, \textit{WebVision} \cite{li2017learning} and \textit{Clothing1M} \cite{xiao2015learning}.  \textit{WebVision} contains 2.4 million images crawled from the websites using the 1,000 concepts in \textit{ImageNet ILSVRC12}. Following the ``mini'' setting in \cite{ma2020normalized,chen2019understanding}, we take the first 50 classes of the Google resized image subset, and evaluate the trained networks on the same 50 classes of the \textit{WebVision} and \textit{ILSVRC12} validation sets, which are exploited as test sets. Meanwhile, \textit{Clothing1M} consists of 1M noisy training examples collected from online shopping websites. 

\myPara{Model and configurations.} For \textit{WebVision}, we use Inception-ResNet V2 \cite{szegedy2017inception}. For \textit{Clothing1M}, we exploit ResNet-50 with ImageNet pretrained weights. The networks are trained using SGD with a momentum of 0.9, a weight decay of 0.001, and a batch size of 32. For \textit{WebVision}, we train the network for 100 epochs. The initial learning rate is set as 0.01 and reduced by a factor of 10 after 50 epochs. For \textit{Clothing1M}, we train the
network for 20 epochs. The initial learning rate is set as 0.001 and reduced by a factor of 10 after 5 epochs. In experiments, we boost DivideMix by applying our methods after the DivideMix paradigm. 

\begin{table}[!h]
    \centering
    \caption{Comparison with state-of-the-art methods trained on (mini) \textit{WebVision}~\cite{chen2019understanding} and \textit{Clothing1M}. 
    The best two results on the \textit{WebVision} validation set, the ImageNet \textit{ILSVRC12} validation set and the \textit{Clothing1M} test set are in \textbf{bold}.}
    \begin{tabular}{l|c|c|c}
    \hline
         Test Dataset & \textit{WebVision} & \textit{ILSVRC12}& \textit{Clothing1M}\\\hline
         Method & Accuracy (\%) & Accuracy (\%) &Accuracy (\%)  \\\hline
         Standard&59.37&58.86 & 68.94\\
         Co-teaching+&61.72&59.72 &64.02\\
         GCE&62.33&61.35 &69.75\\
         SL&62.17&60.95 & 71.02\\
         LRT&62.96&62.09 &71.74\\\hline
         PLC&63.90&62.74 &74.02\\\hline
         MDDC&\textbf{65.06}&\textbf{64.30} &\textbf{74.39}\\
         CDDC&\textbf{64.82}&\textbf{64.09}&\textbf{74.43}\\\hline
    \end{tabular}
	\label{tab:webvision_clothing1m}
\end{table}

\begin{table}[!h]
    \centering
    \caption{Comparison with state-of-the-art method DivideMix~\cite{li2020dividemix} trained on real-world noisy datasets. Here, DivideMix-M and DivideMix-C mean that DivideMix is boosted with our methods MDDC and CDDC respectively. The best two results are in \textbf{bold}.}
    \begin{tabular}{l|c|c|c}
    \hline
    Test Dataset&\textit{WebVision} & \textit{ILSVRC12} & \textit{Clothing1M}\\\hline
         Method & Accuracy (\%) &  Accuracy (\%) & Accuracy (\%) \\\hline
         DivideMix & 77.32&75.20&74.76\\\hline
         DivideMix-M& \textbf{77.45}&\textbf{75.30}&\textbf{74.85}\\
         DivideMix-C& \textbf{77.62}&\textbf{75.68}&\textbf{74.96}\\\hline
    \end{tabular}
	\label{tab:comp_dividemix}
\end{table}

\myPara{Results.} First, Table~\ref{tab:webvision_clothing1m} shows the results on \textit{WebVision} and \textit{Clothing1M}. Compared MDDC and CDDC with the baselines without multiple techniques, we can see our methods outperform them clearly. Second, from Table~\ref{tab:comp_dividemix}, we observe that with our methods, the performance of the state-of-the-art method, i.e., DivideMix, is further enhanced. All results support our claims well. 
\vspace{-10pt}
\subsection{Ablation Study}\label{sec:4.4}
\myPara{The role of different parts.} We study the effect of removing distribution-calibration components to provide insights into what makes our methods successful. Note that without the mean-based and covariance-based methods, MDDC and CDDC will degenerate to the label correction method PLC \cite{zhang2021learningwith}. We analyze the results in Table~\ref{tab:plc} as follows. As we can see, our methods, \textit{i.e.}, MDDC and CDDC, achieve clear improvements in classification performance over PLC.

\begin{table}[!ht]
    \centering
    \footnotesize
    \caption{Mean and standard deviations of test accuracy (\%) on two synthetic noisy datasets. The improvement of MDDC and CDDC over PLC is highlighted.}
    \setlength{\tabcolsep}{1.3mm}{
    \begin{tabular}{c|l|ccc}
    \hline
         Dataset & Noise& PLC & MDDC & CDDC \\\hline
         \parbox[t]{2mm}{\multirow{6}{*}{\rotatebox{90}{\textit{CIFAR-10}}}}&{$\text{Type-\uppercase\expandafter{\romannumeral1}}$ (35\%)} &82.80$\pm$0.27 & 83.12$\pm$0.44$_{ \textcolor{red}{\uparrow0.32}}$ & 83.17$\pm$0.43$_{\textcolor{red}{\uparrow0.37}}$\\
         &{$\text{Type-\uppercase\expandafter{\romannumeral1}}$ (70\%)}& 42.74$\pm$2.14  & 43.70$\pm$0.91$_{\textcolor{red}{\uparrow0.96}}$ & 45.14$\pm$1.11$_{\textcolor{red}{\uparrow2.40}}$\\\cline{2-5}
         &{$\text{Type-\uppercase\expandafter{\romannumeral2}}$ (35\%)} & 81.54$\pm$0.47 & 80.73$\pm$0.19$_{\textcolor{blue}{\downarrow0.81}}$ & 81.11$\pm$0.25$_{\textcolor{blue}{\downarrow0.43}}$\\
         &{$\text{Type-\uppercase\expandafter{\romannumeral2}}$ (70\%)}& 46.04$\pm$2.20 & 47.06$\pm$0.41$_{\textcolor{red}{\uparrow1.02}}$ & 48.10$\pm$1.74$_{\textcolor{red}{\uparrow2.06}}$\\\cline{2-5}
         &{$\text{Type-\uppercase\expandafter{\romannumeral3}}$ (35\%)}& 81.50$\pm$0.50 & 82.34$\pm$0.23$_{\textcolor{red}{\uparrow0.80}}$ & 81.67$\pm$0.07$_{\textcolor{red}{\uparrow0.17}}$\\
         &{$\text{Type-\uppercase\expandafter{\romannumeral3}}$ (70\%)} & 45.05$\pm$1.13 & 46.31$\pm$0.92$_{\textcolor{red}{\uparrow1.26}}$ & 47.92$\pm$1.28$_{\textcolor{red}{\uparrow2.06}}$\\\hline\hline
         \parbox[t]{2mm}{\multirow{6}{*}{\rotatebox{90}{\textit{CIFAR-100}}}}&{$\text{Type-\uppercase\expandafter{\romannumeral1}}$ (35\%)} & 60.01$\pm$0.43 & 63.12$\pm$0.40$_{\textcolor{red}{\uparrow3.11}}$ & 62.89$\pm$0.40$_{\textcolor{red}{\uparrow2.88}}$\\
         &{$\text{Type-\uppercase\expandafter{\romannumeral1}}$ (70\%)}& 45.92$\pm$0.61 & 48.32$\pm$0.85$_{\textcolor{red}{\uparrow2.40}}$ & 47.78$\pm$0.53$_{\textcolor{red}{\uparrow1.86}}$\\\cline{2-5}
         &{$\text{Type-\uppercase\expandafter{\romannumeral2}}$ (35\%)} & 63.68$\pm$0.29 &63.84$\pm$0.23$_{\textcolor{red}{\uparrow0.16}}$ & 64.57$\pm$0.41$_{\textcolor{red}{\uparrow0.89}}$\\
         &{$\text{Type-\uppercase\expandafter{\romannumeral2}}$ (70\%)} & 45.03$\pm$0.50 & 48.41$\pm$0.96$_{\textcolor{red}{\uparrow3.38}}$ & 47.74$\pm$0.62$_{\textcolor{red}{\uparrow2.71}}$\\\cline{2-5}
         &{$\text{Type-\uppercase\expandafter{\romannumeral3}}$ (35\%)}& 63.68$\pm$0.29 & 63.60$\pm$0.33$_{\textcolor{blue}{\downarrow0.08}}$ & 63.46$\pm$0.23$_{\textcolor{blue}{\downarrow0.22}}$\\
         &{$\text{Type-\uppercase\expandafter{\romannumeral3}}$ (70\%)} & 44.45$\pm$0.62  & 47.54$\pm$0.17$_{\textcolor{red}{\uparrow3.09}}$ & 47.04$\pm$0.34$_{\textcolor{red}{\uparrow2.59}}$\\\hline
         \end{tabular}}
    \label{tab:plc}
\end{table}

\myPara{The influence of $\alpha$.} For CDDC, the value of hyperparameter $\alpha$ controls the degree of the disturbance and can be determined with a noisy validation set. Results have shown the effectiveness of the determination. Here, Figure~\ref{fig:ablation_alpha} shows the influence of $\alpha$ with different values. Specifically, the value of hyperparameter $\alpha$ is chosen in $\{0.1, 0.2, 0.3, 0.4\}$. 

\begin{figure}[htbp]
\centering
\subfigure[Type-\uppercase\expandafter{\romannumeral1} + 30\% Sym. noise]{
\begin{minipage}[t]{0.47\linewidth}
\centering
\includegraphics[width=1.7in]{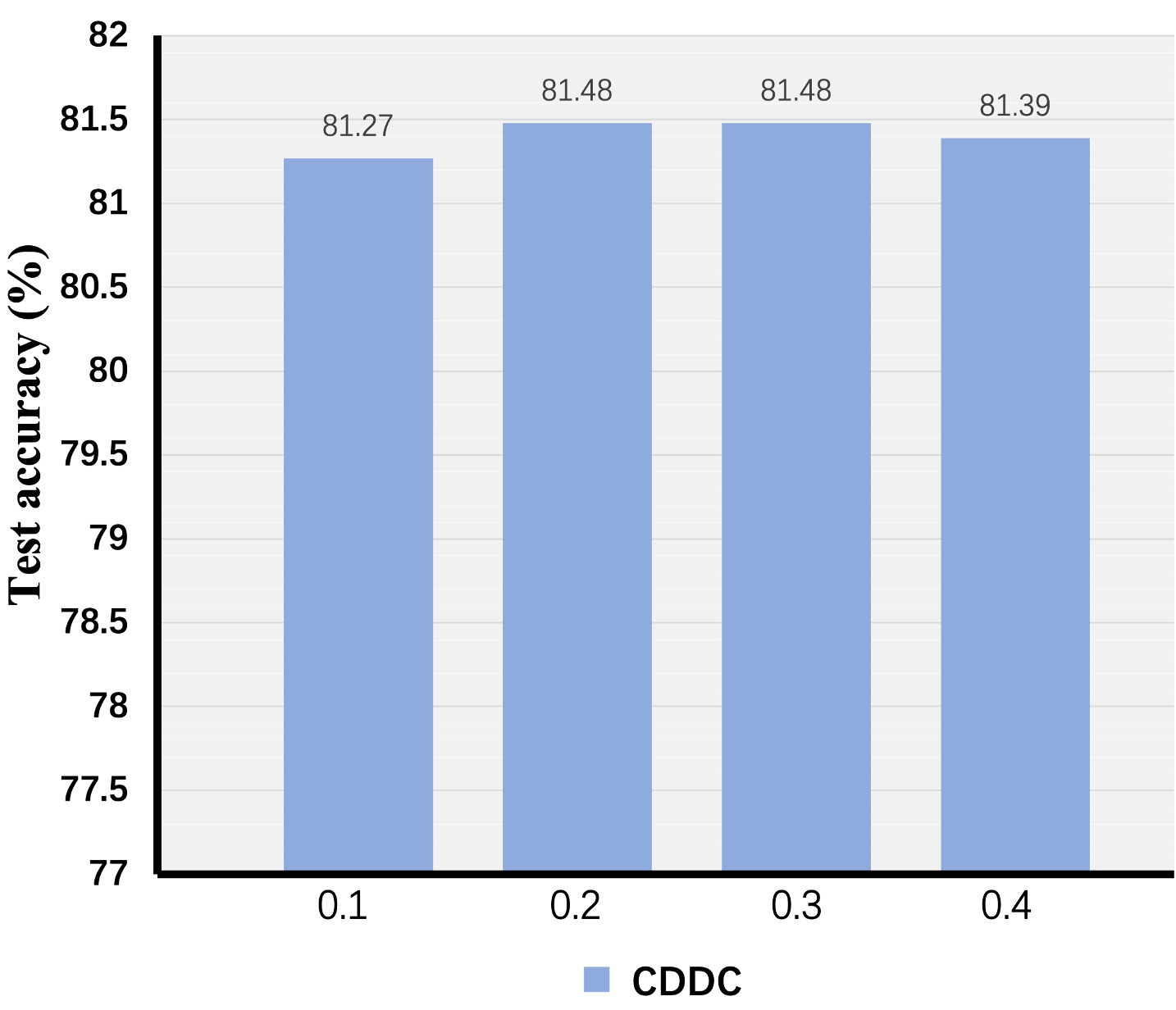}
\end{minipage}}
\hspace{5pt}
\subfigure[Type-\uppercase\expandafter{\romannumeral1} + 60\% Sym. noise]{
\begin{minipage}[t]{0.47\linewidth}
\centering
\includegraphics[width=1.7in]{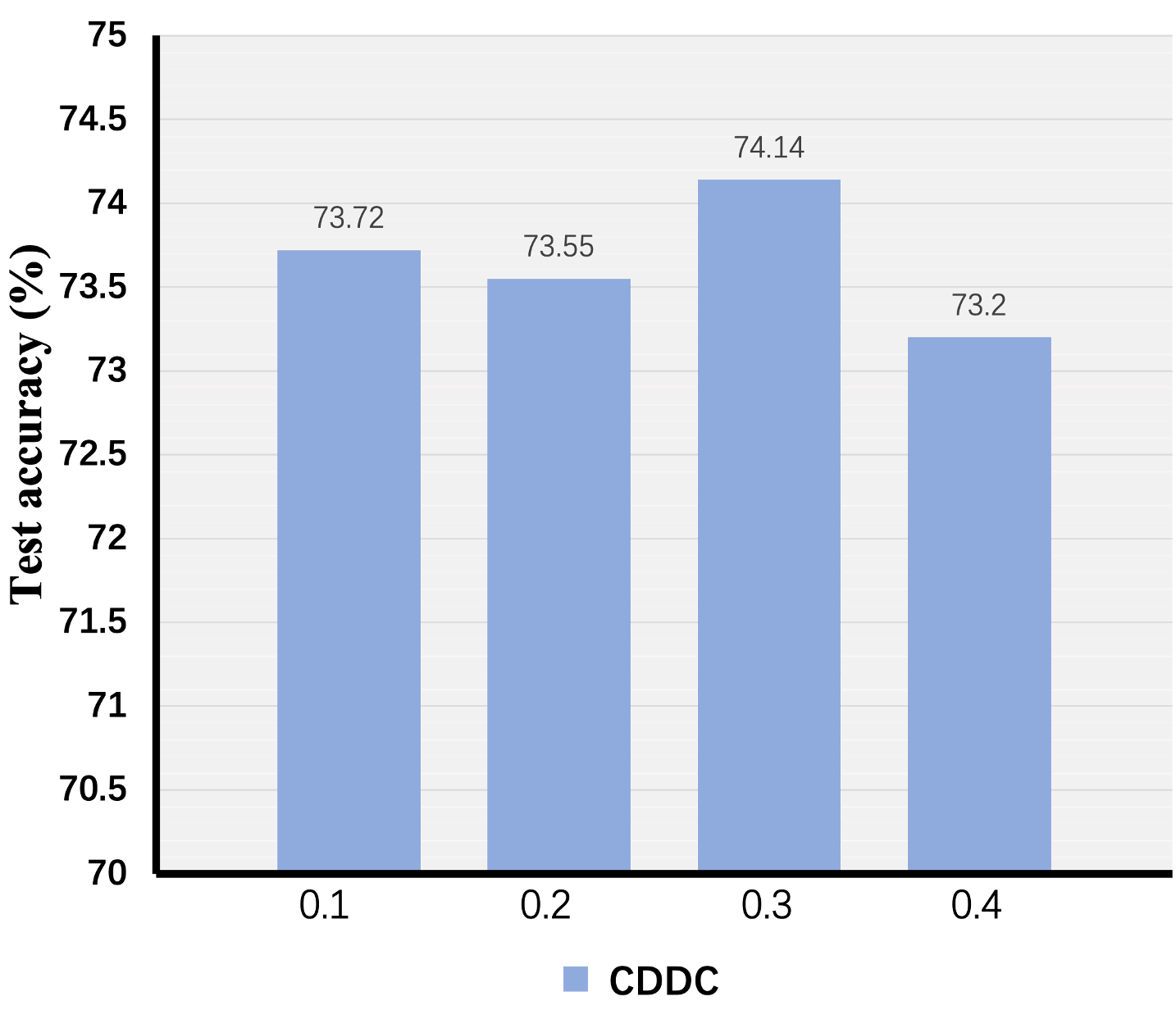}
\end{minipage}}%
\centering

\caption{Results of ablation study on noisy  \textit{CIFAR-1O}  under different values of $\alpha$.  }
\centering

\label{fig:ablation_alpha}

\end{figure}

\myPara{The influence of $\lambda$.} Recall that we sample $\lambda$ proportion data points from the built multivariate
Gaussian distributions. Here, Figure~\ref{fig:ablation_lambda} shows the influence of $\lambda$ with different values. Specifically, the value of hyperparameter $\lambda$ is chosen in $\{0.10, 0.15, 0.20, 0.25\}$. 

\begin{figure}[htbp]
\centering
\subfigure[Type-\uppercase\expandafter{\romannumeral1} + 30\% Sym. noise]{
\begin{minipage}[t]{0.47\linewidth}
\centering
\includegraphics[width=1.7in]{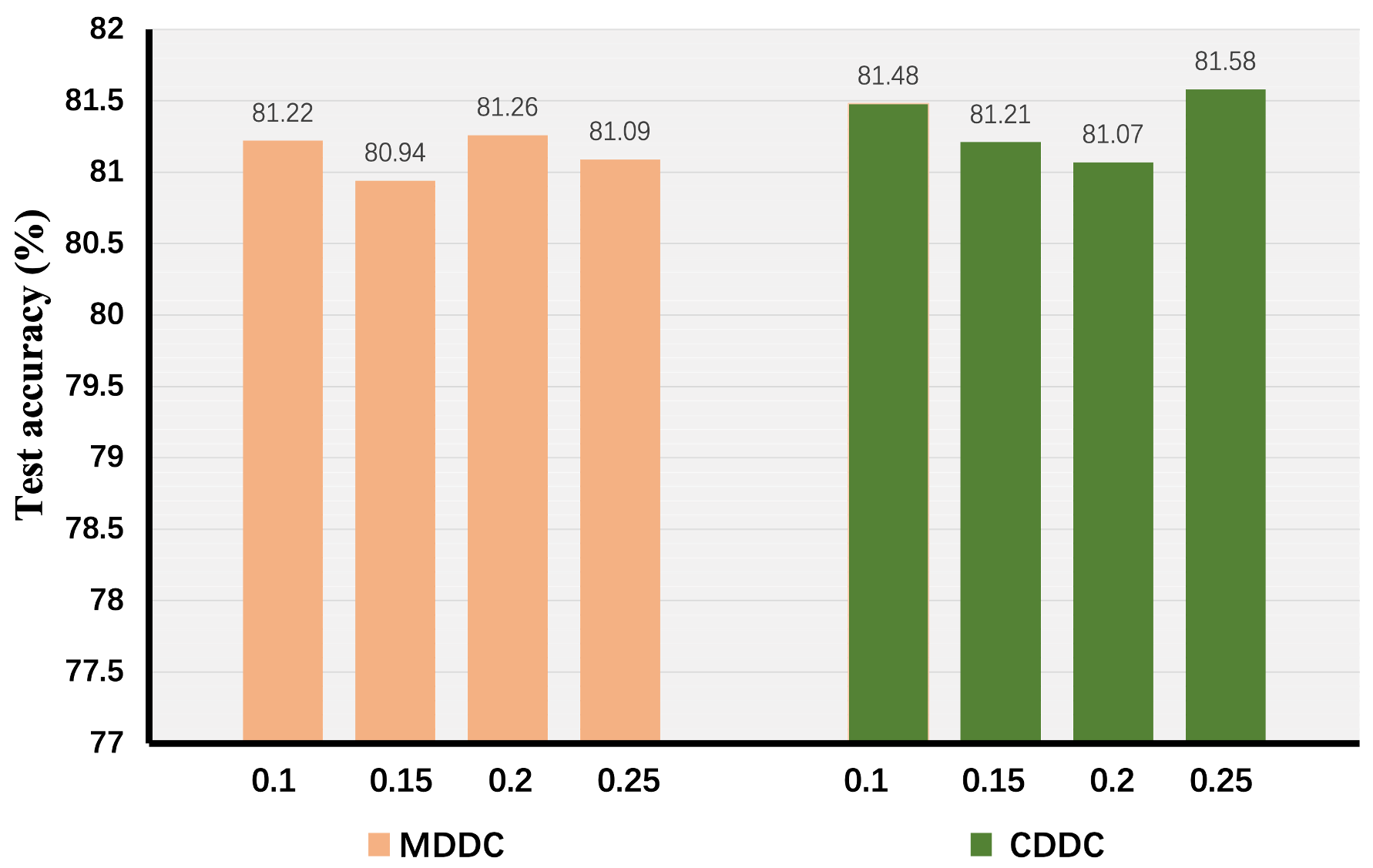}
\end{minipage}}
\hspace{5pt}
\subfigure[Type-\uppercase\expandafter{\romannumeral1} + 60\% Sym. noise]{
\begin{minipage}[t]{0.47\linewidth}
\centering
\includegraphics[width=1.7in]{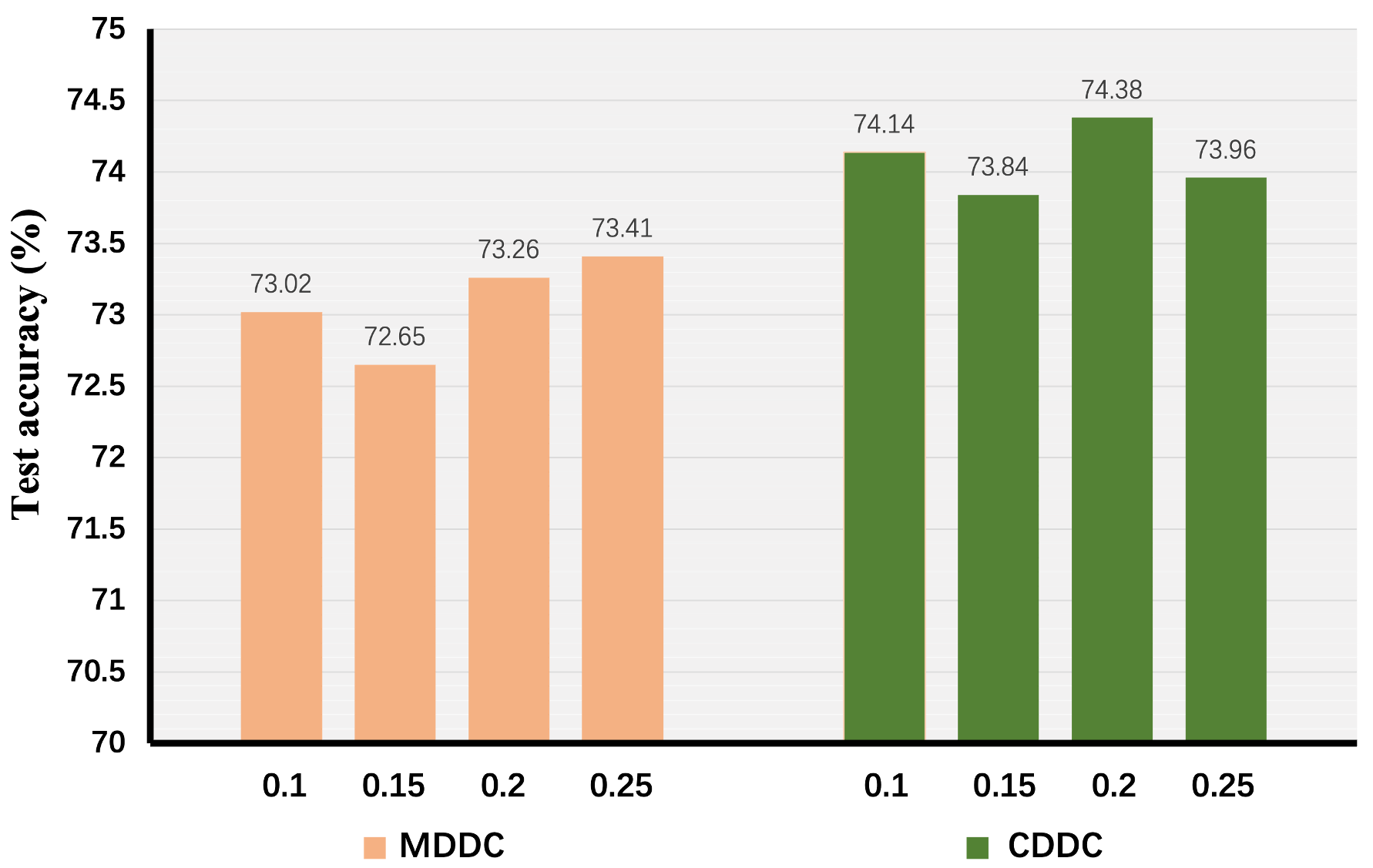}
\end{minipage}}%
\centering

\caption{Results of ablation study on noisy  \textit{CIFAR-1O} under differnet values of $\lambda$. }
\label{fig:ablation_lambda}

\end{figure}

\begin{table}[!ht]
    \centering
    \footnotesize
    \caption{Mean and standard deviations of test accuracy (\%) on noisy \textit{CIFAR-10} with different networks. The improvement of MDDC and CDDC over PLC is highlighted.}
    \setlength{\tabcolsep}{1.3mm}{
    \begin{tabular}{c|l|ccc}
    \hline
    Network & Noise & PLC &MDDC &CDDC\\\hline
    \multirow{3}{*}{\text{SENet18}}&{$\text{Type-\uppercase\expandafter{\romannumeral1}}$+30\% Sym.}    & 78.45$\pm$0.43  &80.88$\pm$0.22$_{\textcolor{red}{\uparrow2.43}}$  & 80.35$\pm$0.15$_{\textcolor{red}{\uparrow1.90}}$\\
    &{$\text{Type-\uppercase\expandafter{\romannumeral1}}$+60\% Sym.}   & 72.56$\pm$0.19& 72.65$\pm$0.95$_{\textcolor{red}{\uparrow0.09}}$ &72.69$\pm$0.53$_{\textcolor{red}{\uparrow0.13}}$\\
     &{$\text{Type-\uppercase\expandafter{\romannumeral1}}$+30\% Asym.}   & 77.33$\pm$0.60  & 78.46 $\pm$1.14$_{\textcolor{red}{\uparrow1.13}}$ &78.52$\pm$0.75$_{\textcolor{red}{\uparrow1.19}}$\\
    \hline
    
    \multirow{3}{*}{\text{MobileNetV2}}&{$\text{Type-\uppercase\expandafter{\romannumeral1}}$+30\% Sym.}  & 74.90$\pm$0.81  & 80.49$\pm$0.20$_{\textcolor{red}{\uparrow5.59}}$ &80.29$\pm$0.17$_{\textcolor{red}{\uparrow5.39}}$\\
    &{$\text{Type-\uppercase\expandafter{\romannumeral1}}$+60\% Sym.}  & 71.35$\pm$1.10  & 72.27$\pm$1.16$_{\textcolor{red}{\uparrow0.92}}$ &73.79$\pm$0.96$_{\textcolor{red}{\uparrow2.44}}$\\
     &{$\text{Type-\uppercase\expandafter{\romannumeral1}}$+30\% Asym.}   & 75.74$\pm$0.82  & 78.23 $\pm$0.32$_{\textcolor{red}{\uparrow2.58}}$ &78.64$\pm$0.34$_{\textcolor{red}{\uparrow2.90}}$\\\hline

     \multirow{3}{*}{\text{WRN-40-2}}&{$\text{Type-\uppercase\expandafter{\romannumeral1}}$+30\% Sym.}  & 72.52$\pm$0.50 &79.30$\pm$0.01$_{\textcolor{red}{\uparrow6.78}}$ &78.80$\pm$0.41$_{\textcolor{red}{\uparrow6.28}}$\\
    &{$\text{Type-\uppercase\expandafter{\romannumeral1}}$+60\% Sym.}  & 70.90$\pm$0.03 & 68.79$\pm$0.09$_{\textcolor{blue}{\downarrow2.11}}$ & 71.41$\pm$0.09$_{\textcolor{red}{\uparrow0.51}}$ \\
     &{$\text{Type-\uppercase\expandafter{\romannumeral1}}$+30\% Asym.}   & 75.01$\pm$0.32 & 78.08$\pm$0.28$_{\textcolor{red}{\uparrow3.07}}$ & 77.91$\pm$0.58$_{\textcolor{red}{\uparrow2.90}}$ \\\hline  
     
     \multirow{3}{*}{\text{EfficientNet}}&{$\text{Type-\uppercase\expandafter{\romannumeral1}}$+30\% Sym.}  & 71.66$\pm$0.20 & 78.09$\pm$0.09$_{\textcolor{red}{\uparrow6.43}}$ & 78.32$\pm$0.04$_{\textcolor{red}{\uparrow6.66}}$ \\
    &{$\text{Type-\uppercase\expandafter{\romannumeral1}}$+60\% Sym.}  & 68.31$\pm$0.14 & 68.21$\pm$0.69$_{\textcolor{blue}{\downarrow0.10}}$ & 70.58$\pm$0.49$_{\textcolor{red}{\uparrow2.27}}$ \\
     &{$\text{Type-\uppercase\expandafter{\romannumeral1}}$+30\% Asym.}   & 74.11$\pm$0.25 & 77.83$\pm$0.14$_{\textcolor{red}{\uparrow3.72}}$ & 77.72$\pm$0.22$_{\textcolor{red}{\uparrow3.61}}$ \\
   \hline
    \end{tabular}
    }
    \label{tab:ablation_networks}
\end{table}

\myPara{The results with different networks.} Before this, on \textit{CIFAR-10} and \textit{CIFAR-100}, we have shown that our methods are effective with a preact ResNet-34 network. To show our methods are robust to different network structures, we exploit SENet18~\cite{hu2018squeeze}, MobileNetV2~\cite{howard2018inverted}, WRN-40-2~\cite{zagoruyko2016wide}, and EfficientNet~\cite{tan2019efficientnet}. The performance of our methods using these networks is shown in Table~\ref{tab:ablation_networks}.

These results show that our methods are not sensitive to the value of $\alpha$ and $\lambda$. Moreover, our methods are robust to the choice of networks. The results imply that it is easy to apply our methods in real-world scenarios.

\section{Conclusion}\label{sec:5}
We propose a dynamic distribution-calibration strategy to handle the distribution shift problem brought by instance-dependent label noise. We suppose that before training data are corrupted by label noise, each class conforms to a multivariate Gaussian distribution at the feature level. Afterwards, two methods named MDDC and CDDC are proposed to calibrate the shifted distribution. Extensive experiments show that our methods can help improve the robustness against complicated instance-dependent label noise. 
\section*{Acknowledgement}
This work was supported by SZSTC Grant No. JCYJ20190809172201639 and WDZC20200820200655001, Shenzhen Key Laboratory\\ ZDSYS20210623092001004.

\newpage
\bibliographystyle{ACM-Reference-Format}
\balance
\bibliography{sample-base}


\begin{thebibliography}{76}


\ifx \showCODEN    \undefined \def \showCODEN     #1{\unskip}     \fi
\ifx \showDOI      \undefined \def \showDOI       #1{#1}\fi
\ifx \showISBNx    \undefined \def \showISBNx     #1{\unskip}     \fi
\ifx \showISBNxiii \undefined \def \showISBNxiii  #1{\unskip}     \fi
\ifx \showISSN     \undefined \def \showISSN      #1{\unskip}     \fi
\ifx \showLCCN     \undefined \def \showLCCN      #1{\unskip}     \fi
\ifx \shownote     \undefined \def \shownote      #1{#1}          \fi
\ifx \showarticletitle \undefined \def \showarticletitle #1{#1}   \fi
\ifx \showURL      \undefined \def \showURL       {\relax}        \fi
\providecommand\bibfield[2]{#2}
\providecommand\bibinfo[2]{#2}
\providecommand\natexlab[1]{#1}
\providecommand\showeprint[2][]{arXiv:#2}

\bibitem[Arpit et~al\mbox{.}(2017)]%
        {arpit2017closer}
\bibfield{author}{\bibinfo{person}{Devansh Arpit},
  \bibinfo{person}{Stanis{\l}aw Jastrz{\k{e}}bski}, \bibinfo{person}{Nicolas
  Ballas}, \bibinfo{person}{David Krueger}, \bibinfo{person}{Emmanuel Bengio},
  \bibinfo{person}{Maxinder~S Kanwal}, \bibinfo{person}{Tegan Maharaj},
  \bibinfo{person}{Asja Fischer}, \bibinfo{person}{Aaron Courville},
  \bibinfo{person}{Yoshua Bengio}, {et~al\mbox{.}}}
  \bibinfo{year}{2017}\natexlab{}.
\newblock \showarticletitle{A closer look at memorization in deep networks}. In
  \bibinfo{booktitle}{\emph{ICML}}. \bibinfo{pages}{233--242}.
\newblock


\bibitem[Berthon et~al\mbox{.}(2021)]%
        {berthon2021idn}
\bibfield{author}{\bibinfo{person}{Antonin Berthon}, \bibinfo{person}{Bo Han},
  \bibinfo{person}{Gang Niu}, \bibinfo{person}{Tongliang Liu}, {and}
  \bibinfo{person}{Masashi Sugiyama}.} \bibinfo{year}{2021}\natexlab{}.
\newblock \showarticletitle{Confidence Scores Make Instance-dependent
  Label-noise learning possible}. In \bibinfo{booktitle}{\emph{ICML}}.
\newblock


\bibitem[Chen et~al\mbox{.}(2019)]%
        {chen2019understanding}
\bibfield{author}{\bibinfo{person}{Pengfei Chen}, \bibinfo{person}{Ben~Ben
  Liao}, \bibinfo{person}{Guangyong Chen}, {and} \bibinfo{person}{Shengyu
  Zhang}.} \bibinfo{year}{2019}\natexlab{}.
\newblock \showarticletitle{Understanding and utilizing deep neural networks
  trained with noisy labels}. In \bibinfo{booktitle}{\emph{ICML}}.
  \bibinfo{pages}{1062--1070}.
\newblock


\bibitem[Chen et~al\mbox{.}(2020)]%
        {chen2020beyond}
\bibfield{author}{\bibinfo{person}{Pengfei Chen}, \bibinfo{person}{Junjie Ye},
  \bibinfo{person}{Guangyong Chen}, \bibinfo{person}{Jingwei Zhao}, {and}
  \bibinfo{person}{Pheng-Ann Heng}.} \bibinfo{year}{2020}\natexlab{}.
\newblock \showarticletitle{Beyond Class-Conditional Assumption: A Primary
  Attempt to Combat Instance-Dependent Label Noise}.
\newblock \bibinfo{journal}{\emph{arXiv preprint arXiv:2012.05458}}
  (\bibinfo{year}{2020}).
\newblock


\bibitem[Cheng et~al\mbox{.}(2020b)]%
        {cheng2020learning}
\bibfield{author}{\bibinfo{person}{Hao Cheng}, \bibinfo{person}{Zhaowei Zhu},
  \bibinfo{person}{Xingyu Li}, \bibinfo{person}{Yifei Gong},
  \bibinfo{person}{Xing Sun}, {and} \bibinfo{person}{Yang Liu}.}
  \bibinfo{year}{2020}\natexlab{b}.
\newblock \showarticletitle{Learning with instance-dependent label noise: A
  sample sieve approach}.
\newblock \bibinfo{journal}{\emph{arXiv preprint arXiv:2010.02347}}
  (\bibinfo{year}{2020}).
\newblock


\bibitem[Cheng et~al\mbox{.}(2020a)]%
        {cheng2017learning}
\bibfield{author}{\bibinfo{person}{Jiacheng Cheng}, \bibinfo{person}{Tongliang
  Liu}, \bibinfo{person}{Kotagiri Ramamohanarao}, {and}
  \bibinfo{person}{Dacheng Tao}.} \bibinfo{year}{2020}\natexlab{a}.
\newblock \showarticletitle{Learning with bounded instance-and label-dependent
  label noise}. In \bibinfo{booktitle}{\emph{ICML}}.
\newblock


\bibitem[Collier et~al\mbox{.}(2021)]%
        {collier2021correlated}
\bibfield{author}{\bibinfo{person}{Mark Collier}, \bibinfo{person}{Basil
  Mustafa}, \bibinfo{person}{Efi Kokiopoulou}, \bibinfo{person}{Rodolphe
  Jenatton}, {and} \bibinfo{person}{Jesse Berent}.}
  \bibinfo{year}{2021}\natexlab{}.
\newblock \showarticletitle{Correlated Input-Dependent Label Noise in
  Large-Scale Image Classification}. In \bibinfo{booktitle}{\emph{Proceedings
  of the IEEE/CVF Conference on Computer Vision and Pattern Recognition}}.
  \bibinfo{pages}{1551--1560}.
\newblock


\bibitem[Diakonikolas and Kane(2019)]%
        {diakonikolas2019recent}
\bibfield{author}{\bibinfo{person}{Ilias Diakonikolas} {and}
  \bibinfo{person}{Daniel~M Kane}.} \bibinfo{year}{2019}\natexlab{}.
\newblock \showarticletitle{Recent advances in algorithmic high-dimensional
  robust statistics}.
\newblock \bibinfo{journal}{\emph{arXiv preprint arXiv:1911.05911}}
  (\bibinfo{year}{2019}).
\newblock


\bibitem[Feng et~al\mbox{.}(2021)]%
        {feng2021can}
\bibfield{author}{\bibinfo{person}{Lei Feng}, \bibinfo{person}{Senlin Shu},
  \bibinfo{person}{Zhuoyi Lin}, \bibinfo{person}{Fengmao Lv},
  \bibinfo{person}{Li Li}, {and} \bibinfo{person}{Bo An}.}
  \bibinfo{year}{2021}\natexlab{}.
\newblock \showarticletitle{Can cross entropy loss be robust to label noise?}.
  In \bibinfo{booktitle}{\emph{IJCAI}}. \bibinfo{pages}{2206--2212}.
\newblock


\bibitem[Hampel et~al\mbox{.}(2011)]%
        {hampel2011robust}
\bibfield{author}{\bibinfo{person}{Frank~R Hampel}, \bibinfo{person}{Elvezio~M
  Ronchetti}, \bibinfo{person}{Peter~J Rousseeuw}, {and}
  \bibinfo{person}{Werner~A Stahel}.} \bibinfo{year}{2011}\natexlab{}.
\newblock \bibinfo{booktitle}{\emph{Robust statistics: the approach based on
  influence functions}}. Vol.~\bibinfo{volume}{196}.
\newblock \bibinfo{publisher}{John Wiley \& Sons}.
\newblock


\bibitem[Han et~al\mbox{.}(2020)]%
        {han2020sigua}
\bibfield{author}{\bibinfo{person}{Bo Han}, \bibinfo{person}{Gang Niu},
  \bibinfo{person}{Xingrui Yu}, \bibinfo{person}{Quanming Yao},
  \bibinfo{person}{Miao Xu}, \bibinfo{person}{Ivor Tsang}, {and}
  \bibinfo{person}{Masashi Sugiyama}.} \bibinfo{year}{2020}\natexlab{}.
\newblock \showarticletitle{Sigua: Forgetting may make learning with noisy
  labels more robust}. In \bibinfo{booktitle}{\emph{ICML}}.
  \bibinfo{pages}{4006--4016}.
\newblock


\bibitem[Han et~al\mbox{.}(2018)]%
        {han2018masking}
\bibfield{author}{\bibinfo{person}{Bo Han}, \bibinfo{person}{Jiangchao Yao},
  \bibinfo{person}{Gang Niu}, \bibinfo{person}{Mingyuan Zhou},
  \bibinfo{person}{Ivor Tsang}, \bibinfo{person}{Ya Zhang}, {and}
  \bibinfo{person}{Masashi Sugiyama}.} \bibinfo{year}{2018}\natexlab{}.
\newblock \showarticletitle{Masking: A new perspective of noisy supervision}.
  In \bibinfo{booktitle}{\emph{NeurIPS}}. \bibinfo{pages}{5836--5846}.
\newblock


\bibitem[Hendrycks et~al\mbox{.}(2018)]%
        {hendrycks2018using}
\bibfield{author}{\bibinfo{person}{Dan Hendrycks}, \bibinfo{person}{Mantas
  Mazeika}, \bibinfo{person}{Duncan Wilson}, {and} \bibinfo{person}{Kevin
  Gimpel}.} \bibinfo{year}{2018}\natexlab{}.
\newblock \showarticletitle{Using trusted data to train deep networks on labels
  corrupted by severe noise}. In \bibinfo{booktitle}{\emph{NeurIPS}}.
\newblock


\bibitem[Hern{\'a}ndez-Lobato et~al\mbox{.}(2014)]%
        {hernandez2014mind}
\bibfield{author}{\bibinfo{person}{Daniel Hern{\'a}ndez-Lobato},
  \bibinfo{person}{Viktoriia Sharmanska}, \bibinfo{person}{Kristian Kersting},
  \bibinfo{person}{Christoph~H Lampert}, {and} \bibinfo{person}{Novi
  Quadrianto}.} \bibinfo{year}{2014}\natexlab{}.
\newblock \showarticletitle{Mind the nuisance: Gaussian process classification
  using privileged noise}. In \bibinfo{booktitle}{\emph{NeurIPS}}.
  \bibinfo{pages}{837--845}.
\newblock


\bibitem[Howard et~al\mbox{.}(2018)]%
        {howard2018inverted}
\bibfield{author}{\bibinfo{person}{Andrew Howard}, \bibinfo{person}{Andrey
  Zhmoginov}, \bibinfo{person}{Liang-Chieh Chen}, \bibinfo{person}{Mark
  Sandler}, {and} \bibinfo{person}{Menglong Zhu}.}
  \bibinfo{year}{2018}\natexlab{}.
\newblock \showarticletitle{Inverted residuals and linear bottlenecks: Mobile
  networks for classification, detection and segmentation}.
\newblock  (\bibinfo{year}{2018}).
\newblock


\bibitem[Hu et~al\mbox{.}(2018)]%
        {hu2018squeeze}
\bibfield{author}{\bibinfo{person}{Jie Hu}, \bibinfo{person}{Li Shen}, {and}
  \bibinfo{person}{Gang Sun}.} \bibinfo{year}{2018}\natexlab{}.
\newblock \showarticletitle{Squeeze-and-excitation networks}. In
  \bibinfo{booktitle}{\emph{Proceedings of the IEEE conference on computer
  vision and pattern recognition}}. \bibinfo{pages}{7132--7141}.
\newblock


\bibitem[Huang et~al\mbox{.}(2019)]%
        {huang2019o2u}
\bibfield{author}{\bibinfo{person}{Jinchi Huang}, \bibinfo{person}{Lie Qu},
  \bibinfo{person}{Rongfei Jia}, {and} \bibinfo{person}{Binqiang Zhao}.}
  \bibinfo{year}{2019}\natexlab{}.
\newblock \showarticletitle{O2u-net: A simple noisy label detection approach
  for deep neural networks}. In \bibinfo{booktitle}{\emph{ICCV}}.
  \bibinfo{pages}{3326--3334}.
\newblock


\bibitem[Huber(1992)]%
        {huber1992robust}
\bibfield{author}{\bibinfo{person}{Peter~J Huber}.}
  \bibinfo{year}{1992}\natexlab{}.
\newblock \showarticletitle{Robust estimation of a location parameter}.
\newblock In \bibinfo{booktitle}{\emph{Breakthroughs in statistics}}.
  \bibinfo{publisher}{Springer}, \bibinfo{pages}{492--518}.
\newblock


\bibitem[Huber(2004)]%
        {huber2004robust}
\bibfield{author}{\bibinfo{person}{Peter~J Huber}.}
  \bibinfo{year}{2004}\natexlab{}.
\newblock \bibinfo{booktitle}{\emph{Robust statistics}}.
  Vol.~\bibinfo{volume}{523}.
\newblock \bibinfo{publisher}{John Wiley \& Sons}.
\newblock


\bibitem[Jiang et~al\mbox{.}(2022)]%
        {jiang2022an}
\bibfield{author}{\bibinfo{person}{Zhimeng Jiang}, \bibinfo{person}{Kaixiong
  Zhou}, \bibinfo{person}{Zirui Liu}, \bibinfo{person}{Li Li},
  \bibinfo{person}{Rui Chen}, \bibinfo{person}{Soo-Hyun Choi}, {and}
  \bibinfo{person}{Xia Hu}.} \bibinfo{year}{2022}\natexlab{}.
\newblock \showarticletitle{An Information Fusion Approach to Learning with
  Instance-Dependent Label Noise}. In \bibinfo{booktitle}{\emph{ICLR}}.
\newblock


\bibitem[Kendall and Gal(2017)]%
        {kendall2017uncertainties}
\bibfield{author}{\bibinfo{person}{Alex Kendall} {and} \bibinfo{person}{Yarin
  Gal}.} \bibinfo{year}{2017}\natexlab{}.
\newblock \showarticletitle{What uncertainties do we need in bayesian deep
  learning for computer vision?}. In \bibinfo{booktitle}{\emph{NeurIPS}}.
\newblock


\bibitem[Kim et~al\mbox{.}(2021)]%
        {kim2021fine}
\bibfield{author}{\bibinfo{person}{Taehyeon Kim}, \bibinfo{person}{Jongwoo Ko},
  \bibinfo{person}{JinHwan Choi}, \bibinfo{person}{Se-Young Yun},
  {et~al\mbox{.}}} \bibinfo{year}{2021}\natexlab{}.
\newblock \showarticletitle{FINE Samples for Learning with Noisy Labels}. In
  \bibinfo{booktitle}{\emph{NeurIPS}}.
\newblock


\bibitem[Kim et~al\mbox{.}(2019)]%
        {kim2019nlnl}
\bibfield{author}{\bibinfo{person}{Youngdong Kim}, \bibinfo{person}{Junho Yim},
  \bibinfo{person}{Juseung Yun}, {and} \bibinfo{person}{Junmo Kim}.}
  \bibinfo{year}{2019}\natexlab{}.
\newblock \showarticletitle{Nlnl: Negative learning for noisy labels}. In
  \bibinfo{booktitle}{\emph{ICCV}}. \bibinfo{pages}{101--110}.
\newblock


\bibitem[Kingma et~al\mbox{.}(2016)]%
        {kingma2016improved}
\bibfield{author}{\bibinfo{person}{Durk~P Kingma}, \bibinfo{person}{Tim
  Salimans}, \bibinfo{person}{Rafal Jozefowicz}, \bibinfo{person}{Xi Chen},
  \bibinfo{person}{Ilya Sutskever}, {and} \bibinfo{person}{Max Welling}.}
  \bibinfo{year}{2016}\natexlab{}.
\newblock \showarticletitle{Improved variational inference with inverse
  autoregressive flow}. In \bibinfo{booktitle}{\emph{NeurIPS}}.
\newblock


\bibitem[Kingma and Welling(2013)]%
        {kingma2013auto}
\bibfield{author}{\bibinfo{person}{Diederik~P Kingma} {and}
  \bibinfo{person}{Max Welling}.} \bibinfo{year}{2013}\natexlab{}.
\newblock \showarticletitle{Auto-encoding variational bayes}.
\newblock \bibinfo{journal}{\emph{arXiv preprint arXiv:1312.6114}}
  (\bibinfo{year}{2013}).
\newblock


\bibitem[Krizhevsky(2009)]%
        {krizhevsky2009learning}
\bibfield{author}{\bibinfo{person}{Alex Krizhevsky}.}
  \bibinfo{year}{2009}\natexlab{}.
\newblock \bibinfo{booktitle}{\emph{Learning multiple layers of features from
  tiny images}}.
\newblock \bibinfo{type}{{T}echnical {R}eport}.
\newblock


\bibitem[Lai et~al\mbox{.}(2016)]%
        {lai2016agnostic}
\bibfield{author}{\bibinfo{person}{Kevin~A Lai}, \bibinfo{person}{Anup~B Rao},
  {and} \bibinfo{person}{Santosh Vempala}.} \bibinfo{year}{2016}\natexlab{}.
\newblock \showarticletitle{Agnostic estimation of mean and covariance}. In
  \bibinfo{booktitle}{\emph{2016 IEEE 57th Annual Symposium on Foundations of
  Computer Science (FOCS)}}. IEEE, \bibinfo{pages}{665--674}.
\newblock


\bibitem[Lee et~al\mbox{.}(2019)]%
        {lee2019robust}
\bibfield{author}{\bibinfo{person}{Kimin Lee}, \bibinfo{person}{Sukmin Yun},
  \bibinfo{person}{Kibok Lee}, \bibinfo{person}{Honglak Lee},
  \bibinfo{person}{Bo Li}, {and} \bibinfo{person}{Jinwoo Shin}.}
  \bibinfo{year}{2019}\natexlab{}.
\newblock \showarticletitle{Robust inference via generative classifiers for
  handling noisy labels}. In \bibinfo{booktitle}{\emph{ICML}}.
  \bibinfo{pages}{3763--3772}.
\newblock


\bibitem[Lee et~al\mbox{.}(2018)]%
        {lee2018cleannet}
\bibfield{author}{\bibinfo{person}{Kuang-Huei Lee}, \bibinfo{person}{Xiaodong
  He}, \bibinfo{person}{Lei Zhang}, {and} \bibinfo{person}{Linjun Yang}.}
  \bibinfo{year}{2018}\natexlab{}.
\newblock \showarticletitle{Cleannet: Transfer learning for scalable image
  classifier training with label noise}. In \bibinfo{booktitle}{\emph{CVPR}}.
  \bibinfo{pages}{5447--5456}.
\newblock


\bibitem[Li et~al\mbox{.}(2020)]%
        {li2020dividemix}
\bibfield{author}{\bibinfo{person}{Junnan Li}, \bibinfo{person}{Richard
  Socher}, {and} \bibinfo{person}{Steven~C.H. Hoi}.}
  \bibinfo{year}{2020}\natexlab{}.
\newblock \showarticletitle{DivideMix: Learning with Noisy Labels as
  Semi-supervised Learning}. In \bibinfo{booktitle}{\emph{ICLR}}.
\newblock


\bibitem[Li et~al\mbox{.}(2022)]%
        {li2022uncertainty}
\bibfield{author}{\bibinfo{person}{Xiaotong Li}, \bibinfo{person}{Yongxing
  Dai}, \bibinfo{person}{Yixiao Ge}, \bibinfo{person}{Jun Liu},
  \bibinfo{person}{Ying Shan}, {and} \bibinfo{person}{Ling-Yu Duan}.}
  \bibinfo{year}{2022}\natexlab{}.
\newblock \showarticletitle{Uncertainty Modeling for Out-of-Distribution
  Generalization}.
\newblock \bibinfo{journal}{\emph{arXiv preprint arXiv:2202.03958}}
  (\bibinfo{year}{2022}).
\newblock


\bibitem[Li et~al\mbox{.}(2021)]%
        {li2021provably}
\bibfield{author}{\bibinfo{person}{Xuefeng Li}, \bibinfo{person}{Tongliang
  Liu}, \bibinfo{person}{Bo Han}, \bibinfo{person}{Gang Niu}, {and}
  \bibinfo{person}{Masashi Sugiyama}.} \bibinfo{year}{2021}\natexlab{}.
\newblock \showarticletitle{Provably End-to-end Label-Noise Learning without
  Anchor Points}. In \bibinfo{booktitle}{\emph{ICML}}.
\newblock


\bibitem[Li et~al\mbox{.}(2017)]%
        {li2017learning}
\bibfield{author}{\bibinfo{person}{Yuncheng Li}, \bibinfo{person}{Jianchao
  Yang}, \bibinfo{person}{Yale Song}, \bibinfo{person}{Liangliang Cao},
  \bibinfo{person}{Jiebo Luo}, {and} \bibinfo{person}{Li-Jia Li}.}
  \bibinfo{year}{2017}\natexlab{}.
\newblock \showarticletitle{Learning from noisy labels with distillation}. In
  \bibinfo{booktitle}{\emph{ICCV}}. \bibinfo{pages}{1910--1918}.
\newblock


\bibitem[Liu et~al\mbox{.}(2020)]%
        {liu2020early}
\bibfield{author}{\bibinfo{person}{Sheng Liu}, \bibinfo{person}{Jonathan
  Niles-Weed}, \bibinfo{person}{Narges Razavian}, {and} \bibinfo{person}{Carlos
  Fernandez-Granda}.} \bibinfo{year}{2020}\natexlab{}.
\newblock \showarticletitle{Early-learning regularization prevents memorization
  of noisy labels}. In \bibinfo{booktitle}{\emph{NeurIPS}}.
\newblock


\bibitem[Liu and Guo(2019)]%
        {liu2019peer}
\bibfield{author}{\bibinfo{person}{Yang Liu} {and} \bibinfo{person}{Hongyi
  Guo}.} \bibinfo{year}{2019}\natexlab{}.
\newblock \showarticletitle{Peer Loss Functions: Learning from Noisy Labels
  without Knowing Noise Rates}.
\newblock \bibinfo{journal}{\emph{arXiv preprint arXiv:1910.03231}}
  (\bibinfo{year}{2019}).
\newblock


\bibitem[Lukasik et~al\mbox{.}(2020)]%
        {lukasik2020does}
\bibfield{author}{\bibinfo{person}{Michal Lukasik}, \bibinfo{person}{Srinadh
  Bhojanapalli}, \bibinfo{person}{Aditya Menon}, {and} \bibinfo{person}{Sanjiv
  Kumar}.} \bibinfo{year}{2020}\natexlab{}.
\newblock \showarticletitle{Does label smoothing mitigate label noise?}. In
  \bibinfo{booktitle}{\emph{ICML}}. \bibinfo{pages}{6448--6458}.
\newblock


\bibitem[Lyu and Tsang(2020)]%
        {lyu2020Curriculum}
\bibfield{author}{\bibinfo{person}{Yueming Lyu} {and} \bibinfo{person}{Ivor~W.
  Tsang}.} \bibinfo{year}{2020}\natexlab{}.
\newblock \showarticletitle{Curriculum Loss: Robust Learning and Generalization
  against Label Corruption}. In \bibinfo{booktitle}{\emph{ICLR}}.
\newblock


\bibitem[Ma et~al\mbox{.}(2020)]%
        {ma2020normalized}
\bibfield{author}{\bibinfo{person}{Xingjun Ma}, \bibinfo{person}{Hanxun Huang},
  \bibinfo{person}{Yisen Wang}, \bibinfo{person}{Simone Romano},
  \bibinfo{person}{Sarah Erfani}, {and} \bibinfo{person}{James Bailey}.}
  \bibinfo{year}{2020}\natexlab{}.
\newblock \showarticletitle{Normalized loss functions for deep learning with
  noisy labels}. In \bibinfo{booktitle}{\emph{ICML}}.
  \bibinfo{pages}{6543--6553}.
\newblock


\bibitem[Mohri et~al\mbox{.}(2018)]%
        {mohri2018foundations}
\bibfield{author}{\bibinfo{person}{Mehryar Mohri}, \bibinfo{person}{Afshin
  Rostamizadeh}, {and} \bibinfo{person}{Ameet Talwalkar}.}
  \bibinfo{year}{2018}\natexlab{}.
\newblock \bibinfo{booktitle}{\emph{Foundations of Machine Learning}}.
\newblock \bibinfo{publisher}{MIT Press}.
\newblock


\bibitem[Nguyen et~al\mbox{.}(2020)]%
        {nguyen2020self}
\bibfield{author}{\bibinfo{person}{Duc~Tam Nguyen},
  \bibinfo{person}{Chaithanya~Kumar Mummadi}, \bibinfo{person}{Thi Phuong~Nhung
  Ngo}, \bibinfo{person}{Thi Hoai~Phuong Nguyen}, \bibinfo{person}{Laura
  Beggel}, {and} \bibinfo{person}{Thomas Brox}.}
  \bibinfo{year}{2020}\natexlab{}.
\newblock \showarticletitle{SELF: Learning to Filter Noisy Labels with
  Self-Ensembling}. In \bibinfo{booktitle}{\emph{ICLR}}.
\newblock


\bibitem[Nishi et~al\mbox{.}(2021)]%
        {nishi2021augmentation}
\bibfield{author}{\bibinfo{person}{Kento Nishi}, \bibinfo{person}{Yi Ding},
  \bibinfo{person}{Alex Rich}, {and} \bibinfo{person}{Tobias Hollerer}.}
  \bibinfo{year}{2021}\natexlab{}.
\newblock \showarticletitle{Augmentation strategies for learning with noisy
  labels}. In \bibinfo{booktitle}{\emph{CVPR}}. \bibinfo{pages}{8022--8031}.
\newblock


\bibitem[Northcutt et~al\mbox{.}(2021)]%
        {northcutt2021confident}
\bibfield{author}{\bibinfo{person}{Curtis Northcutt}, \bibinfo{person}{Lu
  Jiang}, {and} \bibinfo{person}{Isaac Chuang}.}
  \bibinfo{year}{2021}\natexlab{}.
\newblock \showarticletitle{Confident learning: Estimating uncertainty in
  dataset labels}.
\newblock \bibinfo{journal}{\emph{Journal of Artificial Intelligence Research}}
   \bibinfo{volume}{70} (\bibinfo{year}{2021}), \bibinfo{pages}{1373--1411}.
\newblock


\bibitem[Ortego et~al\mbox{.}(2021)]%
        {ortego2021multi}
\bibfield{author}{\bibinfo{person}{Diego Ortego}, \bibinfo{person}{Eric Arazo},
  \bibinfo{person}{Paul Albert}, \bibinfo{person}{Noel~E O'Connor}, {and}
  \bibinfo{person}{Kevin McGuinness}.} \bibinfo{year}{2021}\natexlab{}.
\newblock \showarticletitle{Multi-objective interpolation training for
  robustness to label noise}. In \bibinfo{booktitle}{\emph{CVPR}}.
  \bibinfo{pages}{6606--6615}.
\newblock


\bibitem[Patrini et~al\mbox{.}(2017)]%
        {patrini2017making}
\bibfield{author}{\bibinfo{person}{Giorgio Patrini},
  \bibinfo{person}{Alessandro Rozza}, \bibinfo{person}{Aditya Krishna~Menon},
  \bibinfo{person}{Richard Nock}, {and} \bibinfo{person}{Lizhen Qu}.}
  \bibinfo{year}{2017}\natexlab{}.
\newblock \showarticletitle{Making deep neural networks robust to label noise:
  A loss correction approach}. In \bibinfo{booktitle}{\emph{CVPR}}.
  \bibinfo{pages}{1944--1952}.
\newblock


\bibitem[Pleiss et~al\mbox{.}(2020)]%
        {pleiss2020identifying}
\bibfield{author}{\bibinfo{person}{Geoff Pleiss}, \bibinfo{person}{Tianyi
  Zhang}, \bibinfo{person}{Ethan~R Elenberg}, {and} \bibinfo{person}{Kilian~Q
  Weinberger}.} \bibinfo{year}{2020}\natexlab{}.
\newblock \showarticletitle{Identifying mislabeled data using the area under
  the margin ranking}. In \bibinfo{booktitle}{\emph{NeurIPS}}.
\newblock


\bibitem[Sachdeva et~al\mbox{.}(2021)]%
        {sachdeva2021evidentialmix}
\bibfield{author}{\bibinfo{person}{Ragav Sachdeva}, \bibinfo{person}{Filipe~R
  Cordeiro}, \bibinfo{person}{Vasileios Belagiannis}, \bibinfo{person}{Ian
  Reid}, {and} \bibinfo{person}{Gustavo Carneiro}.}
  \bibinfo{year}{2021}\natexlab{}.
\newblock \showarticletitle{Evidentialmix: Learning with combined open-set and
  closed-set noisy labels}. In \bibinfo{booktitle}{\emph{WACV}}.
  \bibinfo{pages}{3607--3615}.
\newblock


\bibitem[Seo et~al\mbox{.}(2019)]%
        {seo2019combinatorial}
\bibfield{author}{\bibinfo{person}{Paul~Hongsuck Seo}, \bibinfo{person}{Geeho
  Kim}, {and} \bibinfo{person}{Bohyung Han}.} \bibinfo{year}{2019}\natexlab{}.
\newblock \showarticletitle{Combinatorial inference against label noise}. In
  \bibinfo{booktitle}{\emph{NeurIPS}}. \bibinfo{pages}{1173--1183}.
\newblock


\bibitem[Sugiyama and Kawanabe(2012)]%
        {sugiyama2012machine}
\bibfield{author}{\bibinfo{person}{Masashi Sugiyama} {and}
  \bibinfo{person}{Motoaki Kawanabe}.} \bibinfo{year}{2012}\natexlab{}.
\newblock \bibinfo{booktitle}{\emph{Machine learning in non-stationary
  environments: Introduction to covariate shift adaptation}}.
\newblock \bibinfo{publisher}{MIT press}.
\newblock


\bibitem[Szegedy et~al\mbox{.}(2017)]%
        {szegedy2017inception}
\bibfield{author}{\bibinfo{person}{Christian Szegedy}, \bibinfo{person}{Sergey
  Ioffe}, \bibinfo{person}{Vincent Vanhoucke}, {and}
  \bibinfo{person}{Alexander~A Alemi}.} \bibinfo{year}{2017}\natexlab{}.
\newblock \showarticletitle{Inception-v4, inception-resnet and the impact of
  residual connections on learning}. In \bibinfo{booktitle}{\emph{AAAI}}.
\newblock


\bibitem[Tan et~al\mbox{.}(2021)]%
        {tan2021co}
\bibfield{author}{\bibinfo{person}{Cheng Tan}, \bibinfo{person}{Jun Xia},
  \bibinfo{person}{Lirong Wu}, {and} \bibinfo{person}{Stan~Z Li}.}
  \bibinfo{year}{2021}\natexlab{}.
\newblock \showarticletitle{Co-learning: Learning from noisy labels with
  self-supervision}. In \bibinfo{booktitle}{\emph{ACM MM}}.
  \bibinfo{pages}{1405--1413}.
\newblock


\bibitem[Tan and Le(2019)]%
        {tan2019efficientnet}
\bibfield{author}{\bibinfo{person}{Mingxing Tan} {and} \bibinfo{person}{Quoc
  Le}.} \bibinfo{year}{2019}\natexlab{}.
\newblock \showarticletitle{Efficientnet: Rethinking model scaling for
  convolutional neural networks}. In \bibinfo{booktitle}{\emph{International
  conference on machine learning}}. PMLR, \bibinfo{pages}{6105--6114}.
\newblock


\bibitem[Tanaka et~al\mbox{.}(2018)]%
        {tanaka2018joint}
\bibfield{author}{\bibinfo{person}{Daiki Tanaka}, \bibinfo{person}{Daiki
  Ikami}, \bibinfo{person}{Toshihiko Yamasaki}, {and} \bibinfo{person}{Kiyoharu
  Aizawa}.} \bibinfo{year}{2018}\natexlab{}.
\newblock \showarticletitle{Joint Optimization Framework for Learning with
  Noisy Labels}. In \bibinfo{booktitle}{\emph{CVPR}}.
\newblock


\bibitem[Thekumparampil et~al\mbox{.}(2018)]%
        {thekumparampil2018robustness}
\bibfield{author}{\bibinfo{person}{Kiran~K Thekumparampil},
  \bibinfo{person}{Ashish Khetan}, \bibinfo{person}{Zinan Lin}, {and}
  \bibinfo{person}{Sewoong Oh}.} \bibinfo{year}{2018}\natexlab{}.
\newblock \showarticletitle{Robustness of conditional GANs to noisy labels}. In
  \bibinfo{booktitle}{\emph{NeurIPS}}. \bibinfo{pages}{10271--10282}.
\newblock


\bibitem[Wang et~al\mbox{.}(2019a)]%
        {wang2019symmetric}
\bibfield{author}{\bibinfo{person}{Yisen Wang}, \bibinfo{person}{Xingjun Ma},
  \bibinfo{person}{Zaiyi Chen}, \bibinfo{person}{Yuan Luo},
  \bibinfo{person}{Jinfeng Yi}, {and} \bibinfo{person}{James Bailey}.}
  \bibinfo{year}{2019}\natexlab{a}.
\newblock \showarticletitle{Symmetric cross entropy for robust learning with
  noisy labels}. In \bibinfo{booktitle}{\emph{Proceedings of the IEEE/CVF
  International Conference on Computer Vision}}. \bibinfo{pages}{322--330}.
\newblock


\bibitem[Wang et~al\mbox{.}(2019b)]%
        {wang2019implicit}
\bibfield{author}{\bibinfo{person}{Yulin Wang}, \bibinfo{person}{Xuran Pan},
  \bibinfo{person}{Shiji Song}, \bibinfo{person}{Hong Zhang},
  \bibinfo{person}{Gao Huang}, {and} \bibinfo{person}{Cheng Wu}.}
  \bibinfo{year}{2019}\natexlab{b}.
\newblock \showarticletitle{Implicit semantic data augmentation for deep
  networks}.
\newblock \bibinfo{journal}{\emph{Advances in Neural Information Processing
  Systems}}  \bibinfo{volume}{32} (\bibinfo{year}{2019}).
\newblock


\bibitem[Wei et~al\mbox{.}(2020)]%
        {wei2020combating}
\bibfield{author}{\bibinfo{person}{Hongxin Wei}, \bibinfo{person}{Lei Feng},
  \bibinfo{person}{Xiangyu Chen}, {and} \bibinfo{person}{Bo An}.}
  \bibinfo{year}{2020}\natexlab{}.
\newblock \showarticletitle{Combating noisy labels by agreement: A joint
  training method with co-regularization}. In \bibinfo{booktitle}{\emph{CVPR}}.
  \bibinfo{pages}{13726--13735}.
\newblock


\bibitem[Wei and Liu(2021)]%
        {wei2021optimizing}
\bibfield{author}{\bibinfo{person}{Jiaheng Wei} {and} \bibinfo{person}{Yang
  Liu}.} \bibinfo{year}{2021}\natexlab{}.
\newblock \showarticletitle{When Optimizing $f$-divergence is Robust with Label
  Noise}. In \bibinfo{booktitle}{\emph{ICLR}}.
\newblock


\bibitem[Wu et~al\mbox{.}(2021b)]%
        {wu2020class2simi}
\bibfield{author}{\bibinfo{person}{Songhua Wu}, \bibinfo{person}{Xiaobo Xia},
  \bibinfo{person}{Tongliang Liu}, \bibinfo{person}{Bo Han},
  \bibinfo{person}{Mingming Gong}, \bibinfo{person}{Nannan Wang},
  \bibinfo{person}{Haifeng Liu}, {and} \bibinfo{person}{Gang Niu}.}
  \bibinfo{year}{2021}\natexlab{b}.
\newblock \showarticletitle{Class2Simi: A Noise Reduction Perspective on
  Learning with Noisy Labels}. In \bibinfo{booktitle}{\emph{ICML}}.
\newblock


\bibitem[Wu et~al\mbox{.}(2021a)]%
        {wu2021ngc}
\bibfield{author}{\bibinfo{person}{Zhi-Fan Wu}, \bibinfo{person}{Tong Wei},
  \bibinfo{person}{Jianwen Jiang}, \bibinfo{person}{Chaojie Mao},
  \bibinfo{person}{Mingqian Tang}, {and} \bibinfo{person}{Yu-Feng Li}.}
  \bibinfo{year}{2021}\natexlab{a}.
\newblock \showarticletitle{Ngc: A unified framework for learning with
  open-world noisy data}. In \bibinfo{booktitle}{\emph{ICCV}}.
  \bibinfo{pages}{62--71}.
\newblock


\bibitem[Xia et~al\mbox{.}(2022a)]%
        {xia2022extended}
\bibfield{author}{\bibinfo{person}{Xiaobo Xia}, \bibinfo{person}{Bo Han},
  \bibinfo{person}{Nannan Wang}, \bibinfo{person}{Jiankang Deng},
  \bibinfo{person}{Jiatong Li}, \bibinfo{person}{Yinian Mao}, {and}
  \bibinfo{person}{Tongliang Liu}.} \bibinfo{year}{2022}\natexlab{a}.
\newblock \showarticletitle{Extended T: Learning with Mixed Closed-set and
  Open-set Noisy Labels}.
\newblock \bibinfo{journal}{\emph{IEEE Transactions on Pattern Analysis and
  Machine Intelligence}} (\bibinfo{year}{2022}).
\newblock


\bibitem[Xia et~al\mbox{.}(2021)]%
        {xia2021robust}
\bibfield{author}{\bibinfo{person}{Xiaobo Xia}, \bibinfo{person}{Tongliang
  Liu}, \bibinfo{person}{Bo Han}, \bibinfo{person}{Chen Gong},
  \bibinfo{person}{Nannan Wang}, \bibinfo{person}{Zongyuan Ge}, {and}
  \bibinfo{person}{Yi Chang}.} \bibinfo{year}{2021}\natexlab{}.
\newblock \showarticletitle{Robust early-learning: Hindering the memorization
  of noisy labels}. In \bibinfo{booktitle}{\emph{ICLR}}.
\newblock


\bibitem[Xia et~al\mbox{.}(2022b)]%
        {xia2021sample}
\bibfield{author}{\bibinfo{person}{Xiaobo Xia}, \bibinfo{person}{Tongliang
  Liu}, \bibinfo{person}{Bo Han}, \bibinfo{person}{Mingming Gong},
  \bibinfo{person}{Jun Yu}, \bibinfo{person}{Gang Niu}, {and}
  \bibinfo{person}{Masashi Sugiyama}.} \bibinfo{year}{2022}\natexlab{b}.
\newblock \showarticletitle{Sample Selection with Uncertainty of Losses for
  Learning with Noisy Labels}. In \bibinfo{booktitle}{\emph{ICLR}}.
\newblock


\bibitem[Xia et~al\mbox{.}(2020)]%
        {xia2020part}
\bibfield{author}{\bibinfo{person}{Xiaobo Xia}, \bibinfo{person}{Tongliang
  Liu}, \bibinfo{person}{Bo Han}, \bibinfo{person}{Nannan Wang},
  \bibinfo{person}{Mingming Gong}, \bibinfo{person}{Haifeng Liu},
  \bibinfo{person}{Gang Niu}, \bibinfo{person}{Dacheng Tao}, {and}
  \bibinfo{person}{Masashi Sugiyama}.} \bibinfo{year}{2020}\natexlab{}.
\newblock \showarticletitle{Part-dependent label noise: Towards
  instance-dependent label noise}. In \bibinfo{booktitle}{\emph{NeurIPS}}.
\newblock


\bibitem[Xia et~al\mbox{.}(2019)]%
        {xia2019anchor}
\bibfield{author}{\bibinfo{person}{Xiaobo Xia}, \bibinfo{person}{Tongliang
  Liu}, \bibinfo{person}{Nannan Wang}, \bibinfo{person}{Bo Han},
  \bibinfo{person}{Chen Gong}, \bibinfo{person}{Gang Niu}, {and}
  \bibinfo{person}{Masashi Sugiyama}.} \bibinfo{year}{2019}\natexlab{}.
\newblock \showarticletitle{Are Anchor Points Really Indispensable in
  Label-Noise Learning?}. In \bibinfo{booktitle}{\emph{NeurIPS}}.
  \bibinfo{pages}{6835--6846}.
\newblock


\bibitem[Xiao et~al\mbox{.}(2015)]%
        {xiao2015learning}
\bibfield{author}{\bibinfo{person}{Tong Xiao}, \bibinfo{person}{Tian Xia},
  \bibinfo{person}{Yi Yang}, \bibinfo{person}{Chang Huang}, {and}
  \bibinfo{person}{Xiaogang Wang}.} \bibinfo{year}{2015}\natexlab{}.
\newblock \showarticletitle{Learning from massive noisy labeled data for image
  classification}. In \bibinfo{booktitle}{\emph{CVPR}}.
  \bibinfo{pages}{2691--2699}.
\newblock


\bibitem[Yang et~al\mbox{.}(2021)]%
        {yang2021free}
\bibfield{author}{\bibinfo{person}{Shuo Yang}, \bibinfo{person}{Lu Liu}, {and}
  \bibinfo{person}{Min Xu}.} \bibinfo{year}{2021}\natexlab{}.
\newblock \showarticletitle{Free Lunch for Few-shot Learning: Distribution
  Calibration}. In \bibinfo{booktitle}{\emph{ICLR}}.
\newblock


\bibitem[Yao et~al\mbox{.}(2020)]%
        {yao2020searching}
\bibfield{author}{\bibinfo{person}{Quanming Yao}, \bibinfo{person}{Hansi Yang},
  \bibinfo{person}{Bo Han}, \bibinfo{person}{Gang Niu}, {and}
  \bibinfo{person}{James Tin-Yau Kwok}.} \bibinfo{year}{2020}\natexlab{}.
\newblock \showarticletitle{Searching to exploit memorization effect in
  learning with noisy labels}. In \bibinfo{booktitle}{\emph{ICML}}.
  \bibinfo{pages}{10789--10798}.
\newblock


\bibitem[Yao et~al\mbox{.}(2021)]%
        {yao2021jo}
\bibfield{author}{\bibinfo{person}{Yazhou Yao}, \bibinfo{person}{Zeren Sun},
  \bibinfo{person}{Chuanyi Zhang}, \bibinfo{person}{Fumin Shen},
  \bibinfo{person}{Qi Wu}, \bibinfo{person}{Jian Zhang}, {and}
  \bibinfo{person}{Zhenmin Tang}.} \bibinfo{year}{2021}\natexlab{}.
\newblock \showarticletitle{Jo-src: A contrastive approach for combating noisy
  labels}. In \bibinfo{booktitle}{\emph{CVPR}}. \bibinfo{pages}{5192--5201}.
\newblock


\bibitem[Yi and Wu(2019)]%
        {yi2019probabilistic}
\bibfield{author}{\bibinfo{person}{Kun Yi} {and} \bibinfo{person}{Jianxin Wu}.}
  \bibinfo{year}{2019}\natexlab{}.
\newblock \showarticletitle{Probabilistic end-to-end noise correction for
  learning with noisy labels}. In \bibinfo{booktitle}{\emph{CVPR}}.
  \bibinfo{pages}{7017--7025}.
\newblock


\bibitem[Yu et~al\mbox{.}(2019)]%
        {yu2019does}
\bibfield{author}{\bibinfo{person}{Xingrui Yu}, \bibinfo{person}{Bo Han},
  \bibinfo{person}{Jiangchao Yao}, \bibinfo{person}{Gang Niu},
  \bibinfo{person}{Ivor~W Tsang}, {and} \bibinfo{person}{Masashi Sugiyama}.}
  \bibinfo{year}{2019}\natexlab{}.
\newblock \showarticletitle{How Does Disagreement Benefit Co-teaching?}. In
  \bibinfo{booktitle}{\emph{ICML}}.
\newblock


\bibitem[Zagoruyko and Komodakis(2016)]%
        {zagoruyko2016wide}
\bibfield{author}{\bibinfo{person}{Sergey Zagoruyko} {and}
  \bibinfo{person}{Nikos Komodakis}.} \bibinfo{year}{2016}\natexlab{}.
\newblock \showarticletitle{Wide residual networks}.
\newblock \bibinfo{journal}{\emph{arXiv preprint arXiv:1605.07146}}
  (\bibinfo{year}{2016}).
\newblock


\bibitem[Zhang et~al\mbox{.}(2021a)]%
        {zhang2021learning}
\bibfield{author}{\bibinfo{person}{Yivan Zhang}, \bibinfo{person}{Gang Niu},
  {and} \bibinfo{person}{Masashi Sugiyama}.} \bibinfo{year}{2021}\natexlab{a}.
\newblock \showarticletitle{Learning Noise Transition Matrix from Only Noisy
  Labels via Total Variation Regularization}. In
  \bibinfo{booktitle}{\emph{ICML}}.
\newblock


\bibitem[Zhang et~al\mbox{.}(2021b)]%
        {zhang2021learningwith}
\bibfield{author}{\bibinfo{person}{Yikai Zhang}, \bibinfo{person}{Songzhu
  Zheng}, \bibinfo{person}{Pengxiang Wu}, \bibinfo{person}{Mayank Goswami},
  {and} \bibinfo{person}{Chao Chen}.} \bibinfo{year}{2021}\natexlab{b}.
\newblock \showarticletitle{Learning with Feature-Dependent Label Noise: A
  Progressive Approach}. In \bibinfo{booktitle}{\emph{ICLR}}.
\newblock


\bibitem[Zhang and Sabuncu(2018)]%
        {zhang2018generalized}
\bibfield{author}{\bibinfo{person}{Zhilu Zhang} {and} \bibinfo{person}{Mert
  Sabuncu}.} \bibinfo{year}{2018}\natexlab{}.
\newblock \showarticletitle{Generalized cross entropy loss for training deep
  neural networks with noisy labels}. In \bibinfo{booktitle}{\emph{NeurIPS}}.
  \bibinfo{pages}{8778--8788}.
\newblock


\bibitem[Zheng et~al\mbox{.}(2020)]%
        {zheng2020error}
\bibfield{author}{\bibinfo{person}{Songzhu Zheng}, \bibinfo{person}{Pengxiang
  Wu}, \bibinfo{person}{Aman Goswami}, \bibinfo{person}{Mayank Goswami},
  \bibinfo{person}{Dimitris Metaxas}, {and} \bibinfo{person}{Chao Chen}.}
  \bibinfo{year}{2020}\natexlab{}.
\newblock \showarticletitle{Error-bounded correction of noisy labels}. In
  \bibinfo{booktitle}{\emph{ICML}}. \bibinfo{pages}{11447--11457}.
\newblock


\bibitem[Zhu et~al\mbox{.}(2021)]%
        {zhu2021second}
\bibfield{author}{\bibinfo{person}{Zhaowei Zhu}, \bibinfo{person}{Tongliang
  Liu}, {and} \bibinfo{person}{Yang Liu}.} \bibinfo{year}{2021}\natexlab{}.
\newblock \showarticletitle{A second-order approach to learning with
  instance-dependent label noise}. In \bibinfo{booktitle}{\emph{CVPR}}.
  \bibinfo{pages}{10113--10123}.
\newblock


\end{thebibliography}

\clearpage
\appendix

\end{document}


\title{Appendix of “Tackling Instance-Dependent Label Noise with Dynamic Distribution Calibration”}



\appendix
\maketitle
\section{Proof of theoretical results}\label{app:proof}

\subsection{Pre-knowledge}
Before presenting our main theoretical results, we first introduce some definitions. 
\begin{definition}[$(\beta, v)$-consistency~\cite{zhang2021learningwith}]\label{def:consistency_beta_v}
Suppose that a set of data $(\bm{x},\tilde{y})$ are sampled from $\tilde{\mathcal{D}}(\bm{x},\tilde{\eta}(\bm{x}))$, where $\tilde{\eta}(\bm{x})$ outputs the noisy posterior probability for $\bm{x}$. Given a hypothesis class $ \mathcal{M}$ with sufficient examples, and $f(\xx)= \mathop{\arg\min}\limits_{m\in \mathcal{M}}  \mathbbm{E}_{(\xx,\tilde{y})\sim\tilde{\mathcal{D}}(\bm{x},\tilde{\eta}(\bm{x}))}Loss(m(x),\tilde{y})$, we define $\mathcal{M}$ is $(\beta, v)$-consistency if
\begin{align}\label{eq:consistency_beta_v}
    &|f(\xx)-\eta(\xx)|\leq\beta\mathbbm{E}_{(\zz, \tilde{y})\sim\tilde{\mathcal{D}}(\zz,\tilde{\eta}(\zz))}\big[\mathbbm{1}_{\{ \tilde{y}_{\zz}\neq\eta^*(\zz)\}}(\zz)\big||\eta(\zz)-\frac{1}{2}|\geq|\eta(\xx)-\frac{1}{2}|\big]+v \notag,
\end{align}
where $\eta(\zz)$ outputs the clean posterior probability for $\zz$ and $\eta^*(\zz)$ is the Bayes optimal classifier for $\zz$. 
\end{definition}
Definition~\ref{def:consistency_beta_v} measures the inconsistency between $f(\cdot)$ and $\eta(\cdot)$. We can approximate the error of the classifier $f$ at $\xx$ by computing the risk of $\eta^*(\cdot)$ at $\bm{z}$ where $\eta^*(\zz)$ are more confident than $\eta^*(\xx)$.

\begin{definition}[$c_*,c^*)$-bounded distribution~\cite{zhang2021learningwith}]
Denote the cumulative distribution function of $|\eta(\xx)-\frac{1}{2}|$ as $R(\cdot)$. For a random variable $o$, $R(o)= \mathbbm{P}(|\eta(\xx)-\frac{1}{2}|\leq o)$ and the corresponding probability density function is $r(o)$. We define the distribution $\mathcal{D}$ is ($c_*,c^*)$-bounded if $0 < c_*\leq r(o) \leq c^*$ for all $0\leq o \leq \frac{1}{2}$. The worst-case density-imbalance ratio of $\mathcal{D}$ is denoted by $\ell:=\frac{c^*}{c_*}$.
\end{definition}

The bounded condition enforces the continuity of the density function. The continuity allows one example to borrow useful information from its (clean) neighborhood region, which can help handle mislabeled data. 

\begin{definition}[Pure ($\tau, f, \eta$)-level set]
We say a set $S(\tau, \eta) :=\big \{ \xx \big||\eta(\xx)-\frac{1}{2}|\geq\tau\big \}$ is pure for the classifier $f$, if $y_{f(\xx)}=\eta^*(\xx)$ for all $\xx\in S(\tau, \eta)$, where $y_{f(\xx)}$ denotes the label
predicted by $f$, \textit{i.e.}, $y_{f(\xx)}=\mathbbm{1}_{\{f(\xx)\geq\frac{1}{2}\}}$.
\end{definition}
The above definition indicates that the classifier $f$ has sufficient discriminating ability in a specific region. We present the following assumption for theoretical results.

\begin{assumption}
    We assume that the given hypothesis class $\mathcal{M}$ is $(\beta, v)$-consistency, and the underlying distribution $\mathcal{D}(\xx, \eta(x))$ satisfied the $(c^*, c_*)$-bounded condition. 
\end{assumption}
\subsection{Proof}
\begin{lemma}[One round purity improvement]\label{lem:one_round}
Suppose that the classifier $f$ owns a pure ($\tau, f, \eta$)-level set where $3\xi v \leq \tau<\frac{1}{2}$, where $\xi<1$ is a constant that depends on the PMD noise. After one round label correction and one round sampling from estimated Gaussian distributions, the training set consists of two parts: the corrected set ($\xx$, $\tilde{y}_{new})$ and the sampling set ($\xx$, $\tilde{y}_{\mathcal{G}}$). Training on these data, $f$ owns a improved pure ($\tau_{new}, f, \eta$)-level set where $\frac{1}{2}-\tau_{new}\geq (1+\frac{\xi v}{\beta \ell})(\frac{1}{2}-\tau) $. 
\end{lemma}

\begin{proof}
As $f$ owns a pure ($\tau, f, \eta$)-level set: $\forall \xx$, $|\eta(\xx) -\frac{1}{2}|\geq \tau$, we have
\begin{displaymath}
     \mathbbm{E}_{\tilde{\mathcal{D}}+\tilde{\mathcal{D}}_\mathcal{G}}\big[\mathbbm{1}_{\{\tilde{y}_{\zz}\neq\eta^*(\zz)\}}(\zz)\big||\eta(\zz)-\frac{1}{2}|\geq|\eta(\xx)-\frac{1}{2}|\big]=0,
\end{displaymath}
where $\tilde{\mathcal{D}}+\tilde{\mathcal{D}}_\mathcal{G}$ denotes the mixture of the distributions $\tilde{\mathcal{D}}$ and $\tilde{\mathcal{D}}_\mathcal{G}$. 

Next, we further consider $|\eta(\xx) -\frac{1}{2}|\geq \tau-\gamma$. Without loss of generality, the two distributions are mixed uniformly. We have

\begin{align}\label{eq:one_round_error}
    & \mathbbm{E}_{\tilde{\mathcal{D}}+\tilde{\mathcal{D}}_\mathcal{G}}\big[\mathbbm{1}_{\{\tilde{y}_{\zz}\neq\eta^*(\zz)\}}(\zz)\big||\eta(\zz)-\frac{1}{2}\big|\geq\big|\eta(\xx)-\frac{1}{2}|\big] \notag\\ 
    &=\frac{1}{2}\mathbbm{P}_{\tilde{\mathcal{D}}}\big[ \tilde{y}_{\zz}\neq\eta^*(\zz)\big||\eta(\zz)-\frac{1}{2}|\geq|\eta(\xx)-\frac{1}{2}|\big]+\frac{1}{2}\mathbbm{P}_{\tilde{\mathcal{D}}_\mathcal{G}}\big[ \tilde{y}_{\zz}\neq\eta^*(\zz)\big||\eta(\zz)-\frac{1}{2}|\geq|\eta(\xx)-\frac{1}{2}|\big] \notag\\
    &\leq \frac{1}{2}\frac{\mathbbm{P}_{\tilde{\mathcal{D}}}\big[ \tilde{y}_{\zz}\neq\eta^*(\zz)\big| |\eta(\zz)-\frac{1}{2}|\geq|\eta(\xx)-\frac{1}{2}|\big]\mathbbm{P}_{\zz}\big[\tau-\gamma\leq|\eta(\zz)-\frac{1}{2}|\leq \tau\big]}{\mathbbm{P}_{\zz}\big[|\eta(\zz)-\frac{1}{2}|\geq|\eta(\xx)-\frac{1}{2}|\big]} \notag\\ 
    &+\frac{1}{2}\frac{\mathbbm{P}_{\tilde{\mathcal{D}}_\mathcal{G}}\big[ \tilde{y}_{\zz}\neq\eta^*(\zz)\big| |\eta(\zz)-\frac{1}{2}|\geq\big|\eta(\xx)-\frac{1}{2}|\big]\mathbbm{P}_{\zz}\big[\tau-\gamma\leq|\eta(\zz)-\frac{1}{2}|\leq \tau\big]}{\mathbbm{P}_{\zz}\big[|\eta(\zz)-\frac{1}{2}|\geq|\eta(\xx)-\frac{1}{2}|\big]}
\end{align}
Based on Assumption 1, we can obtain
\begin{equation}\label{eq:noise_ratio}
\frac{\mathbbm{P}_{\zz}\big[\tau-\gamma\leq|\eta(\zz)-\frac{1}{2}|\leq \tau\big]}{\mathbbm{P}_{\zz}
    \big[|\eta(\zz)-\frac{1}{2}|\geq|\eta(\xx)-\frac{1}{2}|]}\leq \frac{c^*\gamma}{c_*(\frac{1}{2}-\tau+\gamma)}
\end{equation}
Furthermore, according to the definition of PMD noise~\cite{zhang2021learningwith}, we have

\begin{equation}\label{eq:d_noise}
  \mathbbm{P}_{\tilde{\mathcal{D}}}\big[ \tilde{y}_{\zz}\neq\eta^*(\zz)\big||\eta(\zz)-\frac{1}{2}|\geq|\eta(\xx)-\frac{1}{2}|\big] 
 \leq  c_1(\frac{1}{2}-\tau+\gamma)^{c_2+1}
\end{equation}
Recall that in this work, we assume that the deep features follow multivariate Gaussian distributions and perform robust estimation on the means of the distributions. Based on \cite{durrant2010compressed}, we first bridge the deep features in multivariate Gaussian distributions and probabilistic interpretations. That is to say, we have 
\begin{equation}\label{eq:bound}
    \mathbbm{P}_{\tilde{\mathcal{D}}}\big[ \tilde{y}_{\zz}\neq\eta^*(\zz)\big]\leq 
  \sum\limits_{k=0}^1\sum\limits_{k' \neq k}^1 \exp\big(-\frac{1}{8}(\bar{\bm{\mu}}^k-\bar{\bm{\mu}}^{k'})^\top\bm{\Sigma}^{-1}(\bar{\bm{\mu}}^k-\bar{\bm{\mu}}^{k'})\phi(\bar{\bm{\Sigma}})\big)+ C\Vert\bar{\bm{\mu}}^k-\bm{\mu}^k\Vert_1,
\end{equation}
where $\bar{\bm{\mu}}^k$ is an empirical estimation of $\bm{\mu}^k$, $\phi(\bar{\bm{\Sigma}})=4\Vert \bar{\bm{\Sigma}}^{-1}\Vert_2\Vert \bar{\bm{\Sigma}}\Vert_2\big(1+\Vert \bar{\bm{\Sigma}}^{-1}\Vert_2\Vert \bar{\bm{\Sigma}}\Vert_2\big)^{-2}$ and $C$ ($C>0)$ is a constant. Eq.~(\ref{eq:bound}) shows that the generalization error is bounded by the error of estimated mean. 

In this work, the algorithm \textit{AgnosticMean}~\cite{lai2016agnostic} is used for robust mean estimation, which obtains $\hat{\bm{\mu}}^k$. As discussed in \cite{lai2016agnostic}, based on the Huber’s contamination model~\cite{huber1992robust},  we have $\Vert\hat{\bm{\mu}}^k-\bm{\mu}^k\Vert_1$ is much less than $\Vert\bar{\bm{\mu}}^k-\bm{\mu}^k\Vert_1$ in high-dimensional cases.  
Benefited from \textit{AgnosticMean}, the Gaussian distributions that are robustly estimated, are closer to underlying Gaussian distributions than the Gaussian distributions that are empirically estimated. Therefore, the data sampled from $\tilde{\mathcal{D}}_\mathcal{G}$ contributes more than the data sampled from $\tilde{\mathcal{D}}$ for generalization. Since the estimation error of $\hat{\bm{\mu}}^k$ in \textit{AgnosticMean} is positively correlated with the noise level, we can then derive

\begin{equation}\label{eq:dg_noise}
    \mathbbm{P}_{\tilde{D}_\mathcal{G}}\big[ \tilde{y}_{\zz}\neq\eta^*(\zz)\big||\eta(\zz)-\frac{1}{2}|\geq|\eta(\xx)-\frac{1}{2}|\big] \leq a \cdot c_1(\frac{1}{2}-\tau+\gamma)^{c_2+1},
\end{equation}
where $0<a<1$. Let $\xi=\frac{1+a}{2}c_1(\frac{1}{2}-\tau+\gamma)^{c_2+1}$, it is easy to verify that $\xi<1$. Moreover, combining Eqs.~(\ref{eq:one_round_error}), (\ref{eq:noise_ratio}), (\ref{eq:d_noise}), and (\ref{eq:dg_noise}), we have 

\begin{equation}
    \mathbbm{E}_{\tilde{\mathcal{D}}+\tilde{\mathcal{D}}_\mathcal{G}}\big[\mathbbm{1}_{\{ \tilde{y}_{\zz}\neq\eta^*(\zz)\}}(\zz)\big||\eta(\zz)-\frac{1}{2}|\geq|\eta(\xx)-\frac{1}{2}|\big]\leq \frac{\xi c^*\gamma}{c_*(\frac{1}{2}-\tau+\gamma)}.
\end{equation}
If $\frac{v}{\ell\beta}(\frac{1}{2}-\tau)\leq\gamma\leq\frac{2v}{\ell\beta}(\frac{1}{2}-\tau)$, based on $(\beta, v)-$ consistency condition, for all $\xx$ $s.t.$ $|\eta(\xx)-\frac{1}{2}|\geq \tau-\gamma$, we have
\begin{align}
    &|f(\xx)-\eta(\xx)|\leq\beta\mathbbm{E}_{(\zz, \tilde{y})\sim\tilde{\mathcal{D}}(\zz,\tilde{\eta}(\zz))}\Big[\mathbbm{1}_{\{ \tilde{y}_{\zz}\neq\eta^*(\zz)\}}(\zz)\Big|\big|\eta(\zz)-\frac{1}{2}\big|\geq\big|\eta(\xx)-\frac{1}{2}\big|\Big]+v \notag \\   
    &\leq \beta \cdot \frac{\xi \ell\frac{2v}{\ell\beta}(\frac{1}{2}-\tau)}{(1+\frac{2v}{\ell\beta})(\frac{1}{2}-\tau)}+ v \leq 3\xi v.
\end{align}
From the above equation, we can obtain that if $\tau\geq3\xi v$, the prediction of $f(x)$ will be the same as $\eta(x)$ and $(\tau -\gamma, f, \eta)$-level set  become pure for $f$. Set $\tau_{new}=\tau - \gamma$, we have $\frac{1}{2}-\tau_{new}\geq (1+\frac{\xi v}{\beta \ell})(\frac{1}{2}-\tau)$. 
\end{proof}

\begin{theorem}\label{thm:main}
    After training on the noisy dataset, the proposed method will return the trained classifier $f$ such that
\begin{equation}
    \mathbb{P}_{\xx\sim\tilde{\mathcal{D}}}[y_{f_{final}(\xx)}=\eta^*(\xx)]\geq 1-3\xi c^*v,
\end{equation}
\end{theorem} 
\begin{proof}
Based on our Lemma~{\ref{lem:one_round}}, Lemmas 2 and 3 of \cite{zhang2021learningwith}, we can directly achieve the above theorem.
\end{proof}
\section{Extending to the multiclass scenario}\label{app:multiclass}
 In the \textit{multi-class} scenario, the procedure of PLC~\cite{zhang2021learningwith} is presented as follows. Let $f_i(\xx)$ be the prediction probability of the label $i$. Denoted the classifier's class prediction by $g_{\xx}$, \textit{i.e.}, $g_{\xx}=\arg\max_i f_i(\xx)$. With a threshold $\tau$, $|f_{g(\xx)}(\xx)-f_{\tilde{y}}(\xx)|>\tau$, we flip $\tilde{y}$ to the prediction by $g_{\xx}$. Namely, we measure the difference between the highest confidence and the confidence on $\tilde{y}$. We repeatedly correct labels and improve the network within the epoch, until no label can be corrected. In the next epoch, we sightly relax the threshold $\tau$ for label correction.

 \section{Baselines}
 In this paper, we compare our methods with five representative baselines, including: (i). Standard; (ii). Co-teaching+; (iii). GCE; (iv). SL; (v). LRT. The details of the above baselines are provided as follows.
 \begin{itemize}[leftmargin=*]
    \item Standard, which uses standard Cross Entropy (CE) loss. The network is simply trained on noisy datasets.
    \item Co-teaching+, which exploits two networks to handle label noise simultaneously. To keep the two networks diverged,  Co-teaching+ introduces a ``update by disagreement'' strategy. That is, after feeding forward and predicting all data, the two networks only select disagreement data and then pick small-loss examples for each other.
    \item GCE, which targets on the design of noise-robust loss functions. Although previous works proved that Mean Absolute Error (MAE) is robust to label noise, it is slow to converge. To address the above problem, GCE combines the advantages of both MAE and CE losses. Concretely, GCE applies a Box-Cox transformation to probabilities which can be deemed as a generalized mixture of MAE and CE losses.
    \item SL, which combines a Reverse Cross Entropy (RCE) together with the CE loss. RCE has been proved to be robust to label noise, which contributes to improving the robustness of deep networks.
    \item LRT, which proposes a label correction algorithm and provides a theoretical error bound. The key idea is to use noisy networks' predictions to refurbish data and improve networks. Examples with high probabilities are regarded as refurbishable and correct them with the current predictions. 
\end{itemize}
\bibliographystyle{ACM-Reference-Format}
\bibliography{sample-base}
